\documentclass[twoside]{article}
\usepackage[accepted]{aistats2025}

\usepackage{tikz}
\usetikzlibrary{positioning,angles,quotes}
\usepackage{subcaption}
\usepackage{graphicx}
\usepackage{float}
\usepackage{amsthm}
\usepackage{environ}
\usepackage{float}
\usepackage{moresize}
\usepackage{svg}
\usepackage{amsmath}
\usepackage{comment}
\usepackage{amssymb}
\usepackage{natbib}
\usepackage{adjustbox}
\usepackage{booktabs}
\usepackage{bm}
\usepackage{collect}
\newcommand{\R}{\mathbb{R}}

\newcommand{\B}{\mathbf{B}}
\newcommand{\PP}{\mathbb{P}}
\newcommand{\X}{\mathcal{X}}
\newcommand{\A}{\mathcal{A}}
\newcommand{\St}{\mathcal{S}}
\newcommand{\Var}{\text{Var}}
\newcommand{\f}{ V_i^{\theta_t}(\rho)}
\newcommand{\gf}{\nabla_\theta V_i^{\theta_t}(\rho)}
\newcommand{\hgf}{\hat \nabla_\theta V_i^{\theta_t}(\rho)}
\newcommand{\gb}{\nabla_\theta  B_\eta^{\theta_t}(\rho)}
\newcommand{\hgb}{\hat \nabla_\theta  B_\eta^{\theta_t}(\rho)}
\usepackage{algorithm,algpseudocode}
\usepackage{environ}
\usepackage{multicol}
\allowdisplaybreaks[4]
\usepackage{hyperref}
\newtheorem{theorem}{Theorem}[section]
\newtheorem{corollary}[theorem]{Corollary}
\newtheorem{lemma}[theorem]{Lemma}

\newtheorem{fact}[theorem]{Fact}

\newtheorem{assumption}[theorem]{Assumption}
\newtheorem{proposition}[theorem]{Proposition}

\newtheorem{definition}[theorem]{Definition}
\newtheorem{exmp}[theorem]{Example}
\newtheorem{remark}[theorem]{Remark}
\newtheorem{prob}[theorem]{Problem}
\usepackage{aistats2025}
\begin{document}

%

%

\twocolumn[

\aistatstitle{A safe exploration approach to constrained Markov decision processes}

\aistatsauthor{ Tingting Ni \And Maryam Kamgarpour }

\aistatsaddress{Sycamore, EPFL Lausanne \And  Sycamore, EPFL Lausanne } ]

\begin{abstract}
We consider discounted infinite-horizon constrained Markov decision processes (CMDPs), where the goal is to find an optimal policy that maximizes the expected cumulative reward while satisfying expected cumulative constraints. Motivated by the application of CMDPs in online learning for safety-critical systems, we focus on developing a model-free and \emph{simulator-free} algorithm that ensures \emph{constraint satisfaction during learning}. To this end, we employ the LB-SGD algorithm proposed in \cite{usmanova2022log}, which utilizes an interior-point approach based on the log-barrier function of the CMDP. Under the commonly assumed conditions of relaxed Fisher non-degeneracy and bounded transfer error in policy parameterization, we establish the theoretical properties of the LB-SGD algorithm. In particular, unlike existing CMDP approaches that ensure policy feasibility only upon convergence, the LB-SGD algorithm guarantees feasibility throughout the learning process and converges to the $\varepsilon$-optimal policy with a sample complexity of $\Tilde{\mathcal{O}}(\varepsilon^{-6})$. Compared to the state-of-the-art policy gradient-based algorithm, C-NPG-PDA \cite{bai2022achieving2}, the LB-SGD algorithm requires an additional $\mathcal{O}(\varepsilon^{-2})$ samples to ensure policy feasibility during learning with the same Fisher non-degenerate parameterization.
\end{abstract}
\section{Introduction}
Reinforcement learning (RL) involves studying sequential decision-making problems, where an agent aims to maximize an expected cumulative reward by interacting with an unknown environment \cite{sutton2018reinforcement}. While RL has achieved impressive success in domains like video games and board games \cite{berner2019dota, silver2016mastering, silver2017mastering}, safety concerns arise when applying RL to real-world problems, such as autonomous driving \cite{fazel2018global}, robotics~\cite{koppejan2011neuroevolutionary, ono2015chance}, and cyber-security \cite{zhang2019non}. Incorporating safety into RL algorithms can be done in various ways~\cite{garcia2015comprehensive}. From a problem formulation perspective, one natural approach to incorporate safety constraints is through the framework of discounted infinite horizon constrained Markov decision processes (CMDPs).\looseness-1

In a CMDP, the agent aims to maximize an expected cumulative reward subject to expected cumulative constraints. The CMDP formulation has a long history \cite{altman1999constrained, puterman2014markov} and has been applied in several realistic scenarios \cite{kalweit2020deep, mirchevska2018high, zang2020cmdp}. Due to its applicability, there has been a growing body of literature in recent years that develops learning-based algorithms for CMDPs, employing both model-free~\cite{bai2022achieving2,bai2022achieving, ding2020natural, ding2022convergence, liu2022policy, xu2021crpo, zeng2022finite} and model-based approaches \cite{agarwal2022regret, hasanzadezonuzy2021model, jayant2022model}.\looseness-1

Existing learning-based approaches to the CMDP problem offer various theoretical guarantees regarding constraint violations. Some of these approaches only ensure constraint satisfaction upon algorithm convergence \cite{ding2020natural, ding2022convergence, liu2022policy, xu2021crpo, zeng2022finite}, bounding the average constraint violation by $\varepsilon$. Others enhance these guarantees by aiming for averaged zero constraint violation~\cite{bai2022achieving2, kalagarla2023safe, wei2022provably, wei2022triple}. 

For the practical deployment of RL algorithms in real-world scenarios, particularly those requiring online tuning, it is important to satisfy the constraints during the learning process \cite{abe2010optimizing}. This property is referred to as \emph{safe exploration} \cite{koller2019learning}. Ensuring constraint satisfaction during learning is challenging, as it limits exploration and also requires a more accurate estimation of model parameters or gradients \cite{vaswani2022near}.\looseness-1

To address safe exploration, model-based methods employ either Gaussian processes to learn system dynamics \cite{koller2019learning, cheng2019end, wachi2018safe, berkenkamp2017safe, fisac2018general} or leverage Lyapunov-based analysis~\cite{chow2018lyapunov, chow2019lyapunov} to ensure safe exploration with high probability. However, these approaches lack guarantees on the performance of the learned policy. An alternative model-based approach, known as the constrained upper confidence RL, offers convergence guarantees and ensures safe exploration with high probability. This algorithm has been applied in both infinite horizon average reward scenarios with known transition dynamics \cite{zheng2020constrained} and finite horizon reward scenarios with unknown transition dynamics \cite{liu2021learning, bura2022dope}. However, in complex environments, accurately modeling system dynamics can be computationally challenging \cite{sutton2018reinforcement}.\looseness-1

Policy gradient (PG) algorithms demonstrate their advantage in handling complex environments in a model-free manner \cite{agarwal2021theory}. They have shown empirical success in solving CMDPs \cite{liang2018accelerated, achiam2017constrained,tessler2018reward, liu2020ipo}. Initial guarantees for safe exploration in CMDPs were provided by \cite{achiam2017constrained}, relying on exact policy gradient information. However, with unknown transition dynamics, we can only estimate the gradient. To address this, several works \cite{ding2020natural,ding2022convergence,xu2021crpo,ding2024last} assume access to a simulator~\cite{koenig1993complexity} (also known as a generative model \cite{azar2012sample}), which provides information about the reward and constraint values associated with any state and action of the CMDP. However, practical RL requires learning in real-world scenarios, where access to a simulator may not be feasible. Theoretically, the analysis becomes significantly more challenging without a simulator \cite{jin2018q}.\looseness-1

Among the above works addressing CMDPs with simulator access, \cite{ding2020natural, ding2022convergence, xu2021crpo} provided theoretical guarantees on bounding the average constraint violation over the iterates by $\varepsilon$. However, ensuring an averaged $\varepsilon$ constraint violation is problematic in safety-critical CMDPs, as constraint values may overshoot in each iteration \cite{stooke2020responsive, calvo2023state}, thereby failing to provide safety guarantees for each policy iterate. To partially mitigate this issue, \cite{ding2024last} proposed an approach that ensures constraint satisfaction and optimality only for the last iterate policy. Without simulator access, \cite{zeng2022finite} provided theoretical guarantees on averaged $\varepsilon$ constraint violation, whereas \cite{bai2022achieving2} strengthened this guarantee to an averaged zero constraint violation, albeit at the cost of requiring an additional {\small$\mathcal{O}(\varepsilon^{-2})$} samples compared to the state-of-the-art policy gradient-based algorithm \cite{ding2020natural}. However, like all the aforementioned works, these approaches do not ensure constraint satisfaction during learning. A summary of these works, their constraint satisfaction and convergence guarantees are provided in Table \ref{table:comparison}.\looseness-1

\begin{table*}[t] 
\centering\caption[]{Sample complexity for achieving $\varepsilon$-optimal objectives with guarantees on constraint violations in stochastic policy gradient-based algorithms, considering various parameterizations for discounted infinite horizon CMDPs. For constraint violation, Averaged \(\mathcal{O}(\varepsilon)\) refers to the average constraint violation being bounded by \(\varepsilon\). Averaged zero strengthens the above bound to 0. RPG-PD \cite{ding2024last} ensures the last iterate's safety. Here, we refer to ``w.h.p" as with high probability.}
\label{table:comparison}
 \resizebox{\textwidth}{!}{\begin{tabular}{llllll}
    \toprule
    \multicolumn{6}{c}{Stochastic policy gradient-based algorithms}    \\
    \cmidrule(r){1-6}
    Parameterization    & Algorithm     & Sample complexity & Constraint violation & Optimality& Generative model\\
    \midrule 
     Softmax & NPG-PD \cite{ding2020natural} & ${\mathcal{O}}(\varepsilon^{-2})$ & Averaged $\mathcal{O}(\varepsilon)$ & Average&$\checkmark$  \\ 
     Softmax & PD-NAC \cite{zeng2022finite} & ${\mathcal{O}}(\varepsilon^{-6})$ & Averaged $\mathcal{O}(\varepsilon)$ & Average&$\times$  \\ 
    Neural softmax(ReLu)& CRPO \cite{xu2021crpo} & ${\mathcal{O}}(\varepsilon^{-6})$ & Averaged ${\mathcal{O}}(\varepsilon)$ & Average&$\checkmark$\\ 
    Log-linear & RPG-PD \cite{ding2024last} & ${\Tilde{\mathcal{O}}}(\varepsilon^{-14})$ & 0 at last iterate & Last iterate&$\checkmark$\\ 
    Fisher non-degenerate & PD-ANPG \cite{mondal2024sample} & $\Tilde{\mathcal{O}}(\varepsilon^{-3})$ & Averaged ${\mathcal{O}}(\varepsilon)$& Average &$\checkmark$\\ 
    Fisher non-degenerate & C-NPG-PDA \cite{bai2022achieving2} & $\Tilde{\mathcal{O}}(\varepsilon^{-4})$ & Averaged zero& Average &$\times$\\ 
    Relaxed Fisher non-degenerate & LB-SGD \cite{usmanova2022log} & $\bm{\Tilde{\mathcal{O}}(\varepsilon^{-6})}$ & \textbf{Safe exploration w.h.p } & Last iterate&$\times$\\
    \bottomrule
  \end{tabular}}
\end{table*}
In the field of constrained optimization, safe exploration has been extensively studied using a Bayesian approach based on Gaussian processes \cite{berkenkamp2017safe,sui2015safe, berkenkamp2021bayesian, amani2019linear}. However, Bayesian optimization algorithms suffer from the curse of dimensionality~\cite{frazier2018tutorial, moriconi2020high, eriksson2021high}, and this challenges their applicability to model-free RL settings, where large state and action spaces are often encountered. To address this limitation, \cite{usmanova2020safe, usmanova2022log} proposed a first-order interior point approach, inspired by \cite{hinder2019polynomial}, that tackles the issue by incorporating constraints into the objective using a log barrier function. While their vanilla non-convex analysis could directly apply to policy gradient-based algorithms, the convergence result would be limited to \(\varepsilon\)-approximate stationary points rather than optimal points using {\small\(\mathcal{O}(\varepsilon^{-7})\)} samples in total. While both \cite{usmanova2022log} and \cite{liu2020ipo} demonstrated the success of the log barrier approach on benchmark continuous control problems, to our knowledge, safe exploration and tight convergence guarantees for the log barrier policy gradient method in a CMDP had not been proven.\looseness-1

Our paper is dedicated to providing provable non-asymptotic convergence guarantees for solving CMDPs while ensuring safe exploration under a simulator-free setting. Our contributions are as follows.\looseness-1
\subsection{Contributions}
\begin{itemize}
\item We employ the LB-SGD algorithm proposed in~\cite{usmanova2022log} and develop an interior-point stochastic policy gradient approach for the CMDP problem. Further, we prove that the last iterate policy is {\small $\mathcal{O}(\sqrt{\varepsilon_{bias}}) + \Tilde{\mathcal{O}}(\varepsilon)$} optimal,\footnote{The notation $\Tilde{\mathcal{O}}(\cdot)$ hides the $\log(\frac{1}{\varepsilon})$ term.} while ensuring safe exploration with high probability, utilizing {\small$\Tilde{\mathcal{O}}(\varepsilon^{-6})$} samples. The term {\small$\varepsilon_{bias}$} represents the function approximation error resulting from the policy class (see Theorem \ref{main}).\looseness-1
\item On the technical aspect, we construct an accurate gradient estimator for the log barrier and establish its local smoothness properties. Furthermore, by incorporating common assumptions on the policy class, including relaxed Fisher non-degeneracy and bounded transfer error \cite{liu2020improved,yuan2022general,heniao2023,ding2022global}, we establish the gradient dominance property for the \emph{log barrier function} (see Lemma \ref{gdm}). This in turn enables us to derive convergence guarantees for the last iterate.\looseness-1
\item We contribute to the understanding of the (relaxed) Fisher non-degeneracy and the bounded transfer error assumptions by narrowing down the classes of policies that satisfy these assumptions (see Facts \ref{fact1}, \ref{fact2}). \looseness-1
\end{itemize}
\subsection{Notations}
For a set {\small $\X$}, {\small $\Delta(\X)$} denotes the probability simplex over {\small $\X$}, while {\small $|\X|$} and {\small $\textbf{Cl}\{\X\}$} denote its cardinality and closure, respectively. For any integer $m$, we set {\small$[m]:=\{1,\dots,m\}$}. Here, $\|\cdot\|$ denotes the Euclidean $\ell_2$-norm for vectors and the operator norm for matrices, respectively. The notation {\small$A \succeq B$ }indicates that the matrix {\small$A-B$} is positive semi-definite. We denote the image space and kernel space of the matrix \(A\) as {\small\(\mathbf{Im}(A)\)} and {\small\(\mathbf{Ker}(A)\)}, respectively. The function {\small$f(x)$} is said to be {\small$M$}-smooth on {\small$\X$} if the inequality {\small$f(x)\le f(y)+\langle\nabla f(y),x-y\rangle+\frac{M}{2}\left\|x-y\right\|^2$} holds {\small$ \forall x,y\in \X$}, and {\small$L$}-Lipschitz continuous on {\small$\X$} if {\small$\left|f(x)-f(y)\right|\le L\left\|x-y\right\|$} holds {\small $\forall x,y\in \X$}. \looseness-1
\section{Problem formulation}
We consider an infinite-horizon discounted constrained Markov decision process (CMDP) defined by the tuple {\small$\{\mathcal{S},\mathcal{A}, P,\allowbreak \rho,\left\{r_i\right\}_{i=0}^{m},\gamma \}$}. Here, {\small$\mathcal{S}$} and {\small$\mathcal{A}$} are the state and action spaces, respectively, {\small$\rho\in\Delta({\mathcal{S}})$} denotes the initial state distribution, and {\small${P}(s^\prime \vert s,a)$} is the probability of transitioning from state $s$ to state $s^\prime$ under action $a$. Additionally, {\small$r_0:\mathcal{S}\times \mathcal{A} \rightarrow [0,1]$} is the reward function, and {\small$r_{i}:\mathcal{S}\times \mathcal{A} \rightarrow [-1,1]$} is the utility function for {\small$i\in[m]$}. We denote the discount factor as {\small $\gamma\in(0,1)$}.\looseness-1

We consider a stationary stochastic policy {\small$\pi: \mathcal{S} \rightarrow \Delta(\mathcal{A})$}, which maps states to probability distributions over actions, and we denote $\Pi$ as the set containing all stationary stochastic policies. We introduce the performance measure {\small\( V_{i}^\pi(\rho):= \mathbb{E}_{\tau \sim \pi} \left[ \sum_{t=0}^{\infty} \gamma^t r_i(s_t, a_t) \right] \)}, which is the infinite horizon discounted total return concerning the function \( r_i \). Here, \( \tau \) denotes a trajectory {\small\( \{ (s_0, a_0, s_1, a_1, \hdots) : s_h \in \mathcal{S}, a_h \in \mathcal{A}, h \in \mathbb{N} \} \)} induced by the initial distribution \( s_0 \sim \rho \), the policy \( a_t \sim \pi(\cdot | s_t) \), and the transition dynamics {\small\( s_{t+1} \sim P(\cdot | s_t, a_t) \)}.\looseness-1

In a CMDP, the objective is to find a policy that maximizes the objective function {\small$V_0^\pi(\rho)$} subject to the constraints {\small$V_i^\pi(\rho)$} for {\small$i\in[m]$}:\looseness-1
\begin{align}
\max_{\pi} V_{0}^{\pi}(\rho) \quad \text{s.t.} \quad V_{i}^\pi(\rho)\ge 0, \quad i \in[m]. \tag{RL-O}\label{ori1}
\end{align}
The choice of optimizing only over stationary policies is justified: it has been shown that the set of all optimal policies for a CMDP includes stationary policies~\cite{altman1999constrained}. We assume the existence of a stationary optimal policy $\pi^*$ for problem \eqref{ori1}, ensured by Slater's condition for the finite action space \cite{altman1999constrained} and by suitable measurability and compactness assumptions for the continuous action space, see \cite[Theorem 3.2]{hernandez2000constrained}. \looseness-1

For large or continuous CMDPs, solving \eqref{ori1} is intractable due to the curse of dimensionality~\cite{sutton1999policy}. The policy gradient method allows us to search for the optimal policy $\pi^*$ within a parameterized policy set {\small$\{\pi_\theta,\,\theta\in\R^d\}$}. For example, we can apply neural softmax parametrization for discrete action space, or Gaussian parameterization for continuous action space. For simplicity, we denote {\small$V_i^{\pi_\theta}(\rho)$} as {\small$V_i^{\theta}(\rho)$}, as it is a function of $\theta$. Due to the policy parameterization, we can reformulate problem~\eqref{ori1} into a constrained optimization problem over the finite-dimensional parameter space, as follows:\looseness-1
\begin{prob}
Consider {\small$\theta\in\R^d$}, and we are solving the following optimization problem:
\begin{align}
\max_{\theta} V_{0}^{\theta}(\rho) \quad \text{s.t.} \quad V_{i}^\theta(\rho)\ge 0, \quad i \in[m]. \tag{RL-P}\label{ori}
\end{align}
Here, the feasible set is denoted as {\small$\Theta:=\{\theta\mid V_i^\theta(\rho)\ge 0,\,i\in[m]\}$}, and the corresponding feasible parameterized policy set is {\small$\Pi_\Theta:=\{\pi_\theta\mid \theta\in \Theta\}$}. \looseness-1
\end{prob}
Due to the parameterization, our parameterized policy set may not cover the entire stochastic policy set. Our goal is to find a policy $\pi_\theta$ that closely approximates the optimal policy $\pi^*$ while ensuring \emph{safe exploration}, as defined below. \looseness-1
\begin{definition}\label{safedef}
Given an algorithm providing a sequence {\small$\{\theta_t\}_{t=0}^T$}, we say it ensures safe exploration if $\forall t\in \{0,\dots, T\}$, we have {\small$\min_{i\in[m]}V_{i}^{\theta_t}(\rho) \ge 0$}.
\end{definition}
\section{Log barrier policy gradient approach}
Our approach to safe exploration is based on maximizing the unconstrained log barrier surrogate of Problem~\eqref{ori}, defined as
\begin{align}
B_\eta^\theta(\rho) := V_{0}^\theta(\rho) + \eta\sum_{i=1}^{m}\log V_{i}^\theta(\rho)\nonumber
\end{align}
where {\small$\eta > 0$}. The log barrier algorithm \cite{liu2020ipo,usmanova2022log} can be summarized as
\begin{align}
\label{eq:updates}
   \theta_{t+1}=\theta_t+\gamma_t \hat{\nabla}_\theta B_\eta^\theta(\rho),
\end{align}
where {\small$\gamma_t$} is the stepsize, and {\small$\hat{\nabla}_\theta B_\eta^\theta(\rho)$} is an estimate of the true gradient {\small$\nabla_\theta B_\eta^\theta(\rho)$}, computed as
\begin{align}
\nabla_\theta B_\eta^\theta(\rho) = \nabla V_{0}^\theta(\rho) + \eta\sum_{i=1}^{m}\frac{\nabla V_{i}^\theta(\rho)}{V_{i}^\theta(\rho)}.\label{gradientlog}
\end{align}
The intuition is that the iterates approach the stationary point of the log barrier function from the interior, thereby ensuring safe exploration. Furthermore, we establish that the stationary points of the log barrier function correspond to approximately optimal points of the CMDP objective, ensuring optimality. \looseness-1

The iteration above is arguably simple; In contrast to the approach in \cite{bai2022achieving2,ding2022convergence,liu2022policy,xu2021crpo,zeng2022finite}, our method eliminates the need for projection, adjustment of the learning rate for a dual variable, and the requirement of a simulator capable of simulating the MDP from any state {\small$s\in\mathcal{S}$} and action {\small$a\in\mathcal{A}$}. However, the challenge lies in fine-tuning the stepsize {\small$\gamma_t$} to ensure safe exploration while maintaining convergence. Next, we address these aspects and formalize the algorithm.\looseness-1
\subsection{Estimating the log barrier gradient}
\label{samplech}
Given that we do not have access to the generative model, we need to estimate the log barrier gradient to implement \eqref{eq:updates}, which by Equation \eqref{gradientlog} it implies that we need to estimate both {\small$V_i^\theta(\rho)$} and its gradient {\small$\nabla_\theta V_i^\theta(\rho)$}. Gradient estimators using Monte Carlo approaches have been addressed in past work~\cite{sutton1999policy,williams1992simple,baxter2001infinite}, and we will apply them to estimate the log barrier gradient as follows.\looseness-1

We sample $n$ truncated trajectories with a fixed horizon $H$, where each trajectory is denoted as {\small$\tau_j := \left(s_t^j, a_t^j, \left\{r_{i}(s_t^j, a_t^j)\right\}_{i=0}^{m}\right)_{t=0}^{H-1}$}, including the corresponding reward and utility functions. The estimator of the value function, denoted as {\small \(\hat{V}_i^\theta(\rho)\)}, is computed as the average value over the sampled trajectories: 
$$\hat{V}_i^\theta(\rho):= \frac{1}{n}\sum_{j=1}^{n}\sum_{t=0}^{H-1}\gamma^t r_i(s_t^j, a_t^j).$$
To estimate the gradient {\small$\nabla_\theta V_i^\theta(\rho)$}, we consider the so-called GPOMDP gradient estimator, computed as:
$$\hat \nabla_\theta V_i^\theta(\rho) := \frac{1}{n}\sum_{j=1}^{n}\sum_{t=0}^{H-1}\sum_{t'=0}^{t}\gamma^t r_i(s_t^j,a_t^j) \nabla_\theta\log\pi_\theta(a_{t'}^j|s_{t'}^j).$$
Note that our work extends to other gradient estimators such as REINFORCE for which error bounds similar to those in Proposition \ref{prosmo} below can be derived. \looseness-1

Bounds on the error in the above estimations can be derived given the following assumption.\looseness-1
\begin{assumption}
\label{smoli} 
The gradient and Hessian of the function $\log \pi_\theta(a|s)$ are bounded, i.e., there exist constants {\small$M_g,M_h>0$} such that {\small$\|\nabla_\theta\log \pi_\theta(a|s)\|\le M_g$} and {\small$\|\nabla_\theta^2\log \pi_\theta(a|s)\|\le M_h$} for all {\small$\theta \in \Theta$}.
\end{assumption}
\begin{remark}
Assumption \ref{smoli} has been widely utilized in the analysis of policy gradient methods \cite{liu2020improved,yuan2022general,ding2022global,Xu2020Sample}. It is satisfied for softmax policy, the log-linear policy with bounded feature vectors \cite[Section 6.1.1]{agarwal2021theory}, as well as Cauchy policy~\cite[Appendix B]{heniao2023} and Gaussian policy with action clipping \cite[Appendix D]{Xu2020Sample}.\looseness-1
\end{remark}
As will be shown in Proposition \ref{prosmo} below, Assumption~\ref{smoli} ensures smoothness and Lipschitz continuity properties for the value functions {\small$V_i^\theta(\rho)$}. It also ensures that the estimators of {\small$V_i^\theta(\rho)$} and its gradient {\small$\nabla_\theta V_i^\theta(\rho)$} have sub-Gaussian tail bounds. These tail bounds provide probabilistic guarantees that the estimators do not deviate significantly from their expected values. Such a result is needed for ensuring  safe exploration and convergence  of the algorithm.\looseness-1
\begin{proposition}\label{prosmo}
Let Assumption \ref{smoli} hold. The following properties hold {\small$\forall i\in\{0,\dots,m\}$} and {\small$\forall \theta\in\Theta$}. \looseness-1

1. {\small$V_i^\theta(\rho)$} are $L$-Lipschitz continuous and $M$-smooth, where 
{\small $L:=\frac{M_g}{(1-\gamma)^2}$} and {\small$M:=\frac{M_g^2+M_h}{(1-\gamma)^2}$}.

2. Let {\small$b^0(H):=\frac{\gamma^{H}}{1-\gamma}$} and {\small$b^1(H):=\frac{M_g\gamma^H}{1-\gamma}\sqrt{\frac{1}{1-\gamma}+H}$}, we have
\begin{align*}
    &\left|V_i^\theta(\rho)-\mathbb{E}\left[\hat V_i^\theta(\rho)\right]\right|\le b^0(H), \\
    &\left\|\nabla_\theta V_i^\theta(\rho)-\mathbb{E}\left[\hat \nabla_\theta V_i^\theta(\rho)\right]\right\|\le b^1(H).
\end{align*}
3. Let {\small $\sigma^0(n) := \frac{\sqrt{2}}{\sqrt{n}(1-\gamma)}$} and {\small $\sigma^1(n) := \frac{2\sqrt{2}M_g}{\sqrt{n}(1-\gamma)^{\frac{3}{2}}}$}. For any {\small $\delta \in (0,1)$}, we have
\vspace{-0.1cm}
\small
\begin{align*}
    \PP\left(\left|\hat{V}^{\theta}_{i}(\rho)-\mathbb{E}\left[\hat V_i^\theta(\rho)\right]\right|\leq \sigma^0(n)\sqrt{\ln{\frac{2}{\delta}}}\right) \geq 1-\delta.
\end{align*}
\normalsize
\vspace{-0.4cm}

Furthermore, if $n \geq 8\ln{\frac{e^{\frac{1}{4}}}{\delta}}$, we have 
\vspace{-0.1cm}
\small
\begin{align*}
    \PP\left(\left\|\hat\nabla_\theta{V}^{\theta}_i(\rho)-\mathbb{E}\left[\hat \nabla_\theta V_i^\theta(\rho)\right]\right\|\leq  \sigma^1(n)\sqrt{\ln{\frac{e^{\frac{1}{4}}}{\delta}}}\right) \geq 1-\delta.
\end{align*}
\normalsize
\vspace{-0.4cm}
\end{proposition}
The first property has been proven in \cite[Lemma 2]{bai2022achieving2} and \cite[Lemma 4.4]{yuan2022general}. The bias bound {$b^1(H)$} has been established in \cite[Lemma 4.5]{yuan2022general}. We prove the remaining properties of the concentration bounds for the value estimators using Hoeffding’s inequality and for the gradient estimators using the vector Bernstein inequality (see Appendix \ref{proofsml}). We note that \cite{papini2022smoothing} also proposed tail bounds for gradient estimators; however, we emphasize that our bound is dimension-free due to the use of the vector Bernstein inequality.

Based on the estimators of the  value function {\small$V_i^\theta(\rho)$} and its gradient {\small$\nabla_\theta V_i^\theta(\rho)$}, we construct the estimator for the log barrier gradient required in  iteration \eqref{eq:updates} as follows:\looseness-1
\begin{align}
\label{eq:lbg_estimator}
    \hat\nabla_\theta B_\eta^{\theta}(\rho) := \hat \nabla_\theta V_0^{\theta}(\rho) + \eta\sum_{i=1}^{m}\frac{\hat \nabla_\theta V_i^{\theta}(\rho)}{\hat V_i^{\theta}(\rho)}.
\end{align}
Based on Proposition \ref{prosmo}, we proceed to establish the sub-Gaussian tail bound for the estimator \eqref{eq:lbg_estimator}.\looseness-1
\begin{lemma}\label{gapdelt}
Let Assumption \ref{smoli} hold. Fix a confidence level $\delta\in(0,1]$ and an accuracy $\varepsilon>0$. By  setting {\small$n = \mathcal{O}\left(\varepsilon^{-2}{(1-\gamma)^{-6}(\min_{i\in[m]}V_i^\theta(\rho))^{-4}}\eta^2\ln{\delta}^{-1}\right)$} and {\small$H=\mathcal{O}\left(\ln (\varepsilon\min_{i\in[m]}V_i^\theta(\rho))^{-1}\right)$}, we have
\begin{align*}
    \PP\left(\|\hat\nabla_\theta B_\eta^{\theta}(\rho)-\nabla_\theta B_\eta^{\theta}(\rho)\|\le\varepsilon\right)\ge 1-\delta.
\end{align*}
\end{lemma}
The proof of Lemma \ref{gapdelt}, along with the explicit forms for $n$ and $H$, can be found in Appendix \ref{proof_gapdelt}. Lemma~\ref{gapdelt} indicates that the sample complexity for obtaining an accurate estimate of the log barrier gradient depends on the distance of the iterates to the boundary, which is of order {\small\(\mathcal{O}((\min_{i\in[m]} V_i^\theta(\rho))^{-4})\)}. This observation will be crucial in deriving the convergence rate and sample complexity of our algorithm. 
\subsection{Tuning the stepsize}
\label{seclog}
The log barrier function is not smooth globally because its gradient becomes unbounded as the iterate approaches the boundary of the feasible domain. However, given a local region, the gradient of the log barrier function has bounded growth since {\small\(V_i^\theta(\rho)\)} is bounded and smooth  in that area. Based on this observation, the LB-SGD algorithm \cite{usmanova2022log} developed a local smoothness constant, {\small\(M_t\)} by assuming access to unbiased estimators of {\small$\{V_i^{\theta_t}(\rho)\}_{i=0}^{m}$}. In our RL setting, we incorporate the biases of the estimators (see Proposition \ref{prosmo}), and accordingly tune the stepsize for staying in the local region as follows. \looseness-1

The local smoothness constant {\small \(\hat M_t\)}, accounting for the biases and variances of the objective values and gradients of {\small$\{V_i^{\theta_t}(\rho)\}_{i=0}^{m}$}, are:
\[
\hat M_{t} := M + \sum_{i=1}^{m} \frac{10M\eta}{\underline{\alpha}_i(t)} + 8\eta \sum_{i=1}^{m} \frac{\left(\overline{\beta}_i(t)\right)^2}{\left(\underline{\alpha}_i(t)\right)^2}.
\]
Here, {\small\(\underline{\alpha}_i(t)\)} represents the lower confidence bound of the constraint function {\small\(V_i^{\theta_t}(\rho)\)}, and {\small \(\overline{\beta}_i(t)\)} denotes the upper confidence bound of {\small $|\langle \gf, {\gb} / {\|\gb\|}\rangle|$}. These confidence bounds are derived from Proposition \ref{prosmo} and detailed in Appendix \ref{step}. To prevent overshooting and ensure that iterations remain within the local region where the estimator is valid, we set the stepsize \(\gamma_t\) as follows:\looseness-1

\vspace{-0.4cm}
\small
\begin{align}
\gamma_t=\min\left\{ \frac{1}{\hat M_{t}},{\min_{i\in[m]}\left\{\frac{\underline{\alpha}_i(t)}{\sqrt{M \underline{\alpha}_i(t)}+2|\overline{\beta}_i(t)|}\right\}}\frac{1}{\|\hat\nabla_x B_\eta^{\theta_t}(\rho)\|}\right \}.\label{gammat}
\end{align}
\normalsize
\vspace{-0.4cm}

Above, the second term inside the minimization corresponds to the region around the current iterate {\small\(\theta_t\)} where the estimator is valid (see Appendix \ref{step}).

With the gradient estimators and stepsize defined, we can now provide the LB-SGD approach in Algorithm~\ref{alg:cap}. In summary, the LB-SGD algorithm implements stochastic gradient ascent on the log barrier function {\small$B_\eta^{\theta_t}(\rho)$} using the sampling scheme provided in Section \ref{samplech}, where $\eta$ controls the optimality of the algorithm's output. If the norm of the estimated gradient is smaller than $\frac{\eta}{2}$, the algorithm terminates. However, if the norm exceeds this threshold, the algorithm proceeds with stochastic gradient ascent, based on the stepsize specified in line 8.
\begin{algorithm}
\caption{LB-SGD}\label{alg:cap}
\begin{algorithmic}[1]
\State \textbf{Input}: Smoothness parameter {\small$M=\frac{M_g^2+M_h}{(1-\gamma)^2}$}, batch size $n$, truncated horizon $H$, number of iterations $T$, confidence bound $\delta\in(0,1)$, $\eta>0$.
\For{$t=0,1,\dots,T-1$}
\State Compute {\small$\hat V_i^{\theta_t}(\rho)$} and {\small$\hat \nabla_\theta V_i^{\theta_t}(\rho)$} using sampling scheme~\ref{samplech}.
\State Compute {\small$\hat\nabla_\theta B_\eta^{\theta_t}(\rho)$} using Eq.\eqref{eq:lbg_estimator}.
\If{{\small$\|\hat\nabla_\theta B_\eta^{\theta_t}(\rho)\| \le \frac{\eta}{2}$}}
\State Break and return $\theta_{\text{break}}$.
\EndIf
\State {\small$\theta_{t+1}=\theta_t+\gamma_t \hat\nabla_\theta B_\eta^{\theta_t}(\rho)$} , where $\gamma_t$ is defined in Eq.\eqref{gammat}.
\EndFor
\State Return $\theta_{\text{out}}$, which can be either $\theta_{\text{break}}$ or $\theta_T$.
\end{algorithmic}
\end{algorithm}

Under appropriate assumptions on the policy parameterization, we prove that Algorithm \ref{alg:cap} can find a policy that is {\small${\mathcal{O}}(\sqrt{\varepsilon_{bias}}) + \Tilde{\mathcal{O}}(\varepsilon)$}-optimal with a sample complexity of {\small$\Tilde{\mathcal{O}}(\varepsilon^{-6})$}. Here, {\small$\varepsilon_{bias}$} represents the transfer error from Assumption \ref{bae}, as detailed in Section \ref{convergence}. We provide an informal statement of the main result of our paper here, and in the next section, we elaborate on the assumptions to formalize and prove this statement. \looseness-1
\begin{theorem}
\label{mainn} 
Under suitable assumptions (see Theorem \ref{main} for the precise statement), Algorithm \ref{alg:cap} has the following properties using {\small$\Tilde{\mathcal{O}}(\varepsilon^{-6})$} samples: 
\begin{enumerate}
    \item Safe exploration (see Definition \ref{safedef}) is satisfied with high probability.
    \item The output policy $\pi_{\theta_{\text{out}}}$ achieves {\small$\mathcal{O}(\sqrt{\varepsilon_{bias}}) + \Tilde{\mathcal{O}}(\varepsilon)$}-optimality with high probability.
\end{enumerate}
\end{theorem}
This result extends the findings of \cite{usmanova2022log} in the optimization setting to the RL setting. Their work focuses on non-convex objective and constraint functions {\small\(V_i^\theta(\rho)\)}, and consequently, only convergence to a stationary point can be guaranteed. In this paper, we further demonstrate that LB-SGD ensures that the last iterate policy converges towards a globally optimal policy while guaranteeing safe exploration. Despite the non-convex nature of the objective and constraint functions in the RL setting~\cite{agarwal2021theory}, our sample complexity aligns with that of  zeroth-order feasible iterate method for solving convex constrained optimization problems, as demonstrated in \cite[Theorem 10]{usmanova2022log}.
\section{Technical analysis of log barrier for CMDPs}
The proof of Theorem \ref{main} is divided into safe exploration analysis (Section \ref{safety}) and convergence analysis (Section \ref{convergence}). In the safe exploration analysis (Proposition \ref{small}), we utilize Slater's condition and an Extended Mangasarian-Fromovitz constraint qualification (MFCQ) assumption (see Assumptions \ref{sl} and~\ref{emf}, respectively) to establish lower bounds on the distance of the iterates from the boundary. For the convergence analysis, we rely on the properties of the policy parametrization. Specifically, under the relaxed Fisher non-degeneracy assumption and the bounded transfer error assumption (see Assumptions~\ref{fn} and~\ref{bae}, respectively), we establish the gradient dominance property of the log barrier function in Lemma \ref{gdm}. This allows us to bound the gap between {\small\(V_0^\theta(\rho)\)} and {\small\(V_0^{\pi^*}(\rho)\)} by the norm of the gradient {\small\(\nabla_\theta B_\eta^\theta(\rho)\)}.

\subsection{Safe exploration of the algorithm}
\label{safety}

\begin{assumption}[Slater’s condition]\label{sl} 
There exist a known starting point {\small$\theta_0 \in \Theta$} and $\nu_s > 0$ such that {\small$ V_{i}^{\theta_0}(\rho)\ge \nu_s$, $\forall i\in[m]$}.
\end{assumption}
Assumption \ref{sl} has been commonly used in the analysis of CMDPs \cite{bai2022achieving2,ding2020natural,liu2021learning,liu2021policy}. It is natural to consider this assumption since without a safe initial policy $\pi_{\theta_0}$, the safe exploration (see Definition \ref{safedef}) cannot be satisfied. 

Our next assumption  requires the existence of a  direction that ensures the iterates too close to the boundary (within a distance of $\nu_{\text{emf}}$), can move away from it (by $\ell$). To formalize it, for $p > 0$, let {\small$\mathbf{B}_p(\theta):=\left\{i\in [m]\,|\, 0<V_{i}^{\theta}(\rho)\le p \right\}$} be the set of constraints indicating that $\theta$ is approximately $p$-close to the boundary.
\begin{assumption}[Extended MFCQ]\label{emf}
 There exist constants $0 < \nu_{\text{emf}} \leq \nu_s$ and $\ell > 0$ such that for any $\theta \in \Theta$, there is a direction $s_\theta \in \mathbb{R}^d$ with {\small$\|s_\theta\| = 1$} satisfying {\small$\langle s_\theta, \nabla V_{i}^{\theta}(\rho) \rangle > \ell$} for all {\small$i \in \mathbf{B}_{\nu_{\text{emf}}}(\theta)$}.
\end{assumption}
Note that the MFCQ assumption \cite{mangasarian1967fritz}, commonly considered in non-convex constrained optimization \cite{muehlebach2022constraints,boob2023stochastic}, corresponds to the above with $l=0$ and $\nu_{emf} = 0$.\footnote{The MFCQ assumption ensures that the Karush-Kuhn-Tucker conditions are necessary optimality conditions  \cite{muehlebach2022constraints,boob2023stochastic}.}  Without Assumption \ref{emf}, the iterates generated by the log barrier approach can be {\small$\mathcal{O}\left(\exp-{\eta}^{-1}\right)$} close to the boundary, as illustrated in an example  in Appendix \ref{disimpactemf}. This, in turn, would require very accurate gradient estimators to ensure safe exploration, resulting in extremely high sample complexity, as inferred from Lemma \ref{gapdelt}.  For further insights, we show the cases in which this assumption is implied by the MFCQ assumption in Appendix \ref{disem}.


We are now ready to establish the safe exploration property of Algorithm \ref{alg:cap}.
\begin{proposition}\label{small}
Let Assumptions~\ref{smoli}, \ref{sl}, and \ref{emf} hold. Fix a confidence level $\delta\in(0,1]$. By setting {\small$\eta \leq \nu_{emf}$, $n = \mathcal{O}\left(\eta^{-4} \ell^{-4m}{(1-\gamma)^{-6-8m} }\ln{\delta}^{-1}\right)$} and {\small$H = \mathcal{O}\left(\ln (\ell\eta)^{-1}\right)$}, we have
\begin{align*}
    \PP\left\{\forall t\in \{0,\dots, T\},\,\min_{i\in[m]} V_i^{\theta_t}(\rho)\ge c\eta\right\}\ge 1-mT\delta
\end{align*}
with $c:=\left(\frac{\ell(1-\gamma)^2}{4M_g(1+\frac{4m}{3})}\right)^m$.
\end{proposition}
The proof of Proposition \ref{small} can be found in Appendix~\ref{psmall}. As the transition dynamics are unknown, Algorithm~\ref{alg:cap} only relies on estimates of the true gradients. Accordingly, Proposition \ref{small} demonstrates that the LB-SGD algorithm ensures safe exploration with high probability. Importantly, it also shows that the iterates remain within a distance of $\Omega(\eta)$ from the boundary. This ensures an upper bound on the sample complexity for an accurate estimate of the log barrier gradient, as inferred from
Lemma \ref{gapdelt}, and consequently helps derive the sample complexity of our algorithm.
\subsection{Convergence and sample complexity}
\label{convergence}
To establish algorithm convergence, we rely on two commonly employed assumptions. The first assumption, relaxed Fisher non-degeneracy (see Assumption~\ref{fn}), ensures that the policy can adequately explore the state-action space. The second assumption,  bounded transfer error (see Assumption \ref{bae}), ensures that our parameterized policy set sufficiently covers the entire stochastic policy set. 

To introduce the relaxed Fisher non-degeneracy assumption, we first define the discounted state-action visitation distribution as {\small$d_\rho^{\theta}(s,a):=(1-\gamma)\sum_{t=0}^{\infty}\gamma^t P(s_t=s,a_t=a)$}. The Fisher information matrix induced by policy $\pi_\theta$ is defined as 
$$F^\theta(\rho):=\mathbb{E}_{(s,a)\sim d_\rho^{\theta}}[\nabla \log\pi_\theta(a|s)\left(\nabla \log\pi_\theta(a|s)\right)^T].$$
\begin{assumption}[Relaxed Fisher non-degeneracy]\label{fn}
For any {\small$\theta\in\Theta$}, there exists a positive constant {\small$\mu_F$} such that the smallest non-zero eigenvalue of {\small$F^\theta(\rho)$} is lower bounded by {\small$\mu_F$}.
\end{assumption}

Assumption \ref{fn} weakens the Fisher non-degeneracy assumption commonly used to prove convergence of policy gradient methods \cite{bai2022achieving2,liu2020improved,yuan2022general,masiha2022stochastic,heniao2023,ding2022global}, which further requires {\small$F^\theta(\rho)$} to be strictly positive definite, as formalized in the following definition.

\begin{definition}[Fisher non-degeneracy]\label{def:fisher_nd}
For any {\small$\theta\in\Theta$}, there exists a positive constant {\small$\mu_F$} such that all eigenvalues of {\small$F^\theta(\rho)$} are lower bounded by {\small$\mu_F$}.
\end{definition}


Despite the importance of the above, there has been only a partial understanding of which policy classes comply with Assumption \ref{fn} and Definition~\ref{def:fisher_nd}. It has been claimed that in the tabular setting, the softmax parameterization fails to satisfy Fisher non-degeneracy, particularly when the policy approaches a deterministic policy \cite{ding2022global}. Our result below complements this understanding.\looseness-1
\begin{fact}\label{fact1}
    1) The softmax parameterization does not satisfy \textbf{Fisher non-degeneracy}. 2) Softmax, log-linear, and neural softmax parameterizations fail to satisfy \textbf{relaxed Fisher non-degeneracy} as the policy approaches determinism. \looseness-1
\end{fact}
The proof of Fact \ref{fact1} can be found in Appendix~\ref{disfn}. Building upon this fact, to ensure softmax, log-linear, and neural softmax parameterizations satisfy relaxed Fisher non-degeneracy, we can constrain the parameter space $\Theta$ to a compact set, preventing the policy from approaching determinism. Note that for continuous action spaces, Gaussian and Cauchy policies satisfy Fisher non-degeneracy (see \cite[Appendix B]{heniao2023}).\looseness-1
While the above assumption ensures sufficient exploration of the policy, the bounded transfer error assumption concerns the richness of the policy class. We define the notion of ``richness" in policy parameterization as follows:
\begin{definition}[Richness of Policy Parameterization]
    Define \(\Pi\) as the closure of all stochastically parameterized policies, denoted by \(\textbf{Cl}\{\pi_\theta\mid\theta\in \mathbb{R}^{d}\}\). If we have another policy parameterization \(\Pi'\), we say \(\Pi'\) is a richer parameterization compared to \(\Pi\) if \(\Pi'\subsetneq\Pi\).
\end{definition}

To formalize the bounded transfer error assumption, we start by defining the state value function as {\small$V^{\theta}_i(s):=\mathbb{E}_{\tau\sim\pi_\theta}[\sum_{t=0}^{\infty}\gamma^t r_i(s_t,a_t)|s_0=s]$}, the state-action value function as {\small$Q^{\theta}_i(s,a):=\mathbb{E}_{\tau\sim\pi_\theta}[\sum_{t=0}^{\infty}\gamma^t r_i(s_t,a_t)|s_0=s,a_0=a]$}, and the advantage function as {\small$A_i^\theta(s,a)=Q_i^\theta(s,a)-V^\theta_i(s)$}. With these definitions in place, the transfer error is defined as {\small$L(\mu_i^*,\theta,d_\rho^{\pi^*}):=\mathbb{E}_{(s,a)\sim d_\rho^{\pi^*}} \!\! \left[\left(A_i^\theta(s,a)-(1-\gamma){\mu_i^*}^T\nabla_\theta\log\pi_\theta(a|s)\right)^2\right]$}, where {\small$\mu_i^*=\left(F^\theta\left(\rho\right)\right)^{-1} \nabla_\theta V_i^\theta(\rho)$}. This formulation is termed as the transfer error because it shows the error in approximating the advantage function {\small$A_i^\theta$}, which depends on {\small$d_\rho^{\theta}$}, while the expectation of the error is taken with respect to a fixed measure {\small$d_\rho^{\pi^*}$}. 
\begin{assumption}[Bounded transfer error]\label{bae}
For any \(\theta \in \Theta\), there exists a non-negative constant \(\varepsilon_{bias}\) such that for {\small\(i \in \{0, \dots, m\}\), \(L(\mu_i^*,\theta,d_\rho^{\pi^*}) \leq \varepsilon_{bias}\)}.
\end{assumption}
Assumption \ref{bae} has been  utilized in several works~\cite{liu2020improved,yuan2022general,heniao2023,ding2022global}. The general understanding is that softmax parameterization results in {\small$\varepsilon_{bias}=0$}. This result is extended by either assuming a very specific class of MDPs, such as a linear MDP model with low-rank transition dynamics \cite{pmlr-v70-jiang17c,jin2020provably,yang2019sample}, or a very specific policy class, such as a ``rich" two-layer neural network \cite{wang2019neural}. Building upon these findings, we present a more general result connecting the richness of policy classes to the transfer error. \looseness-1

\begin{fact}\label{fact2}
For log-linear and neural softmax policy parameterizations, a richer policy parameterization leads to a decrease in the transfer error {\small\(\varepsilon_{bias}\)}.
\end{fact}
We have provided the exact formalization of ``richness" of a parametrization and the proof of Fact \ref{fact2} in Appendix \ref{disbae}. 

With these assumptions in place, we can establish bounds on the optimality of the policy {\small$\pi_\theta$} given {\small$\nabla_\theta B_\eta^\theta(\rho)$} in the following lemma.
\begin{lemma}\label{gdm}
Let Assumptions \ref{smoli}, \ref{fn}, and \ref{bae} hold. For any {\(\theta \in \Theta\),} we have
\begin{align*}
   V_{0}^{\pi^*}\left(\rho\right)-V_{0}^{\theta}\left(\rho\right)\le& m\eta+ \frac{\sqrt{\varepsilon_{bias}}}{1-\gamma}\left(1+\sum_{i\in[m]}\frac{\eta}{ V_{i}^{\theta}\left(\rho\right)}\right)\\
   &+ \frac{M_h}{\mu_F}\left\| \nabla_\theta B_\eta^{\theta}(\rho)\right\|. 
\end{align*}
\end{lemma}
The proof of Lemma \ref{gdm} is provided in Appendix~\ref{proofgdm}. Previous works such as \cite{ding2022global,masiha2022stochastic} established gradient dominance of the value function in the MDP setting, bounding the optimality gap for {\small\(V_0^\theta(\rho)\)} by the norm of its gradient {\small\(\nabla_\theta V_0^\theta(\rho)\)}. Here, we establish this property in the CMDP setting by bounding the optimality gap using the norm of the log barrier gradient, {\small\( \nabla_\theta B_\eta^\theta(\rho) \)}, along with an additional term {\small\( \mathcal{O}\left(\sum_{i\in[m]} {\eta \sqrt{\varepsilon_{\text{bias}}}}/{V_i^\theta(\rho)}\right) \)}. Therefore, to bound the sub-optimality of the stationary point of the log barrier function, we need to provide a lower bound on the distance of the stationary point from the boundary, which is shown in the following lemma.
\begin{lemma}
\label{stationl}
Let Assumptions \ref{smoli} and \ref{emf} hold. For any stationary point \( \theta_{\text{st}} \) of the log barrier function, we have
\[
\min_{i\in[m]}\left\{ V_i^{\theta_{\text{st}}}(\rho) \right\} \ge \min\left\{{\nu_{\text{emf}}},\, \frac{\min\{\eta, \nu_{\text{emf}}\} \ell}{mL} \right\}.
\]
\end{lemma}
The proof of Lemma \ref{stationl} can be found in Appendix~\ref{station}. From the above lemma, we prove that the stationary points of the log barrier function are at most {\small\( \Omega(\nu_{\text{emf}} + \eta) \)} close to the boundary. Combining this with Lemma \ref{gdm}, we conclude that the stationary points are {\small\( \mathcal{O}(\eta + \sqrt{\varepsilon_{\text{bias}}} \max\{1, {\eta}/{\nu_{\text{emf}}}\}) \)}-optimal. Meanwhile, the gradient ascent method ensures convergence to the stationary point of {\small\( B_\eta^\theta(\rho) \)}. Leveraging Lemmas~\ref{small} and \ref{gdm}, we complete the proof of Theorem~\ref{mainn}. Below, we provide the precise statement of Theorem~\ref{mainn}.
\begin{theorem}\label{main} 
Let Assumptions~\ref{smoli},~\ref{sl},~\ref{emf},~\ref{fn}, and~\ref{bae} hold. Fix a confidence level $\delta\in(0,1]$, by setting $n$ and $H$ as in Proposition \ref{small}, after $T$ iterations of the Algorithm \ref{alg:cap}, the following holds:
\begin{enumerate}
    \item {\small $\PP\left( \forall t\in \{0,\dots,T\},\,\min_{i\in[m]}V_{i}^{\theta_t}(\rho) \ge 0\right) \ge 1-mT\delta$}.
    \item With a probability of at least {\small$1-mT\delta$}, the output policy $\pi_{\theta_{\text{out}}}$ satisfies
    
    \vspace{-0.5cm}
    \begin{small}
        \begin{align}
            &V_0^{\pi^*}(\rho) - V_0^{ \theta_{\text{out}}}(\rho)\le \mathcal{O}(\frac{\sqrt{\varepsilon_{\text{bias}}}}{\ell^m(1-\gamma)^{2m+1}}) + {\mathcal{O}}(\frac{\eta}{\mu_F})\nonumber\\
            &+{\mathcal{O}}(\eta\ln\frac{1}{\eta})+\mathcal{O}(\exp{(-C\mu_F T \eta^2)})(V_{0}^{\pi^*} (\rho)- V_{0}^{\theta_{0}}(\rho)),\label{red}
        \end{align}
        \end{small} 
    \vspace{-0.5cm}
    
    where {\small$C:=\frac{c}{2L^2(1+\frac{m}{c})\max\left\{4+\frac{5Mc}{L^2},1+\sqrt{\frac{Mc}{4L^2}}\right\}}$}.
    \end{enumerate}
\end{theorem}
The proof of Theorem \ref{main} can be found in Appendix~\ref{proofmain}. Inequality \eqref{red} yields the following key insights: the last iterate of the algorithm converges to the neighborhood of the optimal point of {\small\(V_{0}^{\pi^*}(\rho)\)} at a rate of {\small\(\mathcal{O}(\exp(-C\mu_F T\eta^2 ))\)}. Hence, larger values of {\(\mu_F\)} lead to faster convergence. On the other hand, the neighborhood's radius is influenced by two factors: {\(\mu_F\)} and the transfer error {\(\varepsilon_{bias}\)}. A smaller {\(\mu_F\)} corresponds to a less random policy, reducing exploration. A larger {\(\varepsilon_{bias}\)} indicates inadequate policy parameterization. Consequently, smaller {\(\mu_F\)} values and larger {\(\varepsilon_{bias}\)} values prevent the algorithm from reaching the optimal policy. Therefore, {\(\mu_F\)} controls optimality and convergence speed.\looseness-1

Based on Theorem \ref{main}, we can determine the sample complexity of the algorithm required to ensure safe exploration and achieve $\varepsilon$-optimality, as stated in the following corollary, whose proof is provided in Appendix \ref{proofcor1}.\looseness-1
\begin{corollary} \label{cor1}
    Given {\small$\varepsilon\!\!>\!\!0$}, the sample complexity of Algorithm~\ref{alg:cap} to return an {\small\({\mathcal{O}}\left(\varepsilon\ln{\varepsilon}^{-1}+{\varepsilon}{\mu_F}^{-1}\right) + \mathcal{O}\left({{\varepsilon_{\text{bias}}}}^{{1}/{2}}\ell^{-m}{(1-\gamma)^{-2m-1}}\right)\)}-optimal policy while ensuring safe exploration with high probability is {\small${\mathcal{O}}\left({\mu_F^{-1}\varepsilon^{-6}\ell^{-6m} (1-\gamma)^{-10-12m} }\ln{\ell}^{-1}\ln^3{\varepsilon}^{-1}\right)$}.
\end{corollary}
When compared to the state-of-the-art policy gradient-based algorithm, C-NPG-PDA \cite{bai2022achieving2}, which only provides guarantees for averaged zero constraint violation assuming Fisher non-degeneracy, our algorithm demands an additional {\small\(\mathcal{O}(\varepsilon^{-2})\)} samples only assuming relaxed Fisher non-degeneracy. This increase in sampling requirement serves as the price for ensuring safe exploration.

\section{Computational Experiment}

While our work primarily focuses on establishing theoretical guarantees for safe exploration in CMDPs, we validate the performance of our algorithm against benchmark algorithms in Appendix \ref{expri}. Since our paper emphasizes the feasibility of iterates rather than bounding average constraint violations, we compare with two typical CMDP learning algorithms: (1) the IPO algorithm \cite{liu2020ipo}, which combines the log barrier method with a policy gradient approach and fixed stepsize, and (2) the RPG-PD algorithm~\cite{ding2024last}, a primal-dual method that guarantees the feasibility of the last iterate, using an entropy-regularized policy gradient and quadratic-regularized gradient ascent for the dual variable. 

As RPG-PD is designed for discrete action spaces, we conducted experiments in a standard gridworld environment with softmax parameterization. Our experiments show that both IPO and RPG-PD ensure safe exploration through manual stepsize tuning, LB-SGD achieves safe exploration using adaptive stepsizes estimated from samples, eliminating the need for tuning.  Moreover, LB-SGD is more sample-efficient in finding the optimal policy compared to RPG-PD, which aligns with theoretical results, as RPG-PD requires \( \mathcal{O}(\varepsilon^{-14}) \) samples to find an \( \mathcal{O}(\varepsilon) \)-optimal policy, while LB-SGD requires only \( \Tilde{\mathcal{O}}(\varepsilon^{-6}) \) samples.\looseness-1

Our experiment confirmed that LB-SGD achieves safe exploration while converges to the optimal policy efficiently. However, as expected, ensuring these guarantees necessitates a higher number of samples per iteration near the boundary for accurate estimates. It would be interesting to determine whether this sample complexity is inherent to our algorithm and its analysis or to the safe exploration requirement. \looseness-1
\section{Conclusion}\label{conclusion}
We employed a log barrier policy gradient approach for ensuring safe exploration in CMDPs. Our work establishes the convergence of the algorithm to an optimal point and characterizes its sample complexity. A potential direction for future research is to explore methods that can further reduce the sample complexity of safe exploration. This could involve incorporating variance reduction techniques, leveraging MDP structural characteristics (e.g., natural policy gradient method), and extending the relaxed Fisher non-degenerate parameterization to general policy representations. Another potential research avenue is to establish lower bounds for the safe exploration problem.

\section{Acknowledgments}
This research is gratefully supported by the Swiss National Science Foundation (SNSF).
\bibliographystyle{apalike}
\bibliography{main}   
\onecolumn
\appendix
\section{Comparison of model-free safe RL algorithms} 
\label{distabel}
Regarding the past work on policy gradient in infinite horizon discounted CMDPs, we further provide details on the notion of constraint satisfaction. To this end, we provide an extended version of Table \ref{table:comparison} to include the assumptions. 
\begin{table*}[t] 
\centering
\caption[]{Sample complexity for achieving $\varepsilon$-optimal objectives with guarantees on constraint violations in stochastic policy gradient-based algorithms, considering various parameterizations for discounted infinite horizon CMDPs.}
 \resizebox{\textwidth}{!}{\begin{tabular}{lllllll}
    \toprule
    \multicolumn{7}{c}{Stochastic policy gradient-based algorithms}    \\
    \cmidrule(r){1-7}
    Parameterization    & Algorithm     & Sample complexity & Constraint violation & Optimality& Generative model&Slater's condition \\
    \midrule 
     Softmax & NPG-PD \cite{ding2020natural} & ${\mathcal{O}}(\varepsilon^{-2})$ &  Averaged $\mathcal{O}(\varepsilon)$ & Average&$\checkmark$ & $\checkmark$ \\ 
      Softmax & PD-NAC \cite{zeng2022finite} & ${\mathcal{O}}(\varepsilon^{-6})$ &  Averaged $\mathcal{O}(\varepsilon)$ & Average&$\times$ & $\checkmark$ \\  
    Neural softmax(ReLu)& CRPO \cite{xu2021crpo} & ${\mathcal{O}}(\varepsilon^{-6})$ & Averaged ${\mathcal{O}}(\varepsilon)$ & Average&$\checkmark$& $\times$\\ 
    Log-linear & RPG-PD \cite{ding2024last} & ${\Tilde{\mathcal{O}}}(\varepsilon^{-14})$ & 0 at last iterate & Last iterate&$\checkmark$& $\checkmark$\\ 
    Fisher non-degenerate & PD-ANPG \cite{mondal2024sample} & $\Tilde{\mathcal{O}}(\varepsilon^{-3})$ &  Averaged ${\mathcal{O}}(\varepsilon)$& Average &$\checkmark$&$\checkmark$\\ 
    Fisher non-degenerate & C-NPG-PDA \cite{bai2022achieving2} & $\Tilde{\mathcal{O}}(\varepsilon^{-4})$ & Averaged zero& Average &$\times$& $\checkmark$\\ 
    Relaxed Fisher non-degenerate & LB-SGD \cite{usmanova2022log} & $\bm{\Tilde{\mathcal{O}}(\varepsilon^{-6})}$ & \textbf{Safe exploration w.h.p } & Last iterate&$\times$& $\checkmark$\\
    \bottomrule
  \end{tabular}}
\end{table*}
\begin{enumerate}
\item (Slater’s condition) Compared to Table \ref{table:comparison}, the above table includes an additional column detailing the assumptions required for convergence analysis. LB-SGD, unlike all the other methods in the table, requires a feasible initial policy since our work is the only one that focuses on safe exploration.
\item (Constraint violation) In our work, we define safe exploration as ensuring constraint satisfaction throughout the learning process, as defined in the property \ref{safedef}. Our LB-SGD algorithm achieves safe exploration with high probability. However, in \cite{bai2022achieving2}, the authors claim to achieve zero constraint violation but employ a different definition, specified as:
\[\frac{1}{T}\sum_{t=0}^{T-1}V_{i}^{\theta_t}(\rho)\ge 0,\,\forall i\in[m].\]
It is important to note that while their algorithm aims for zero constraint violation, individual iterates during the learning process may still violate the constraints. Hence, we refer to it as an averaged zero constraint violation, since safe exploration represents a stronger notion of constraint violation guarantees. Additionally, in \cite{ding2024last}, the regularized policy gradient primal-dual (RPG-PD) algorithm returns the last iterate policy satisfying the constraints, but it does not provide guarantees for safe exploration.
\item (Sample complexity) In the constraint-rectified policy optimization (CRPO) algorithm \cite{xu2021crpo}, the authors provide a general result for measuring the algorithm's performance in \cite[Theorem 2]{xu2021crpo}. We conclude that $\mathcal{O}(\varepsilon^{-6})$ is the optimal sample complexity for achieving an $\mathcal{O}(\varepsilon)$ optimality gap for the CRPO algorithm.
\end{enumerate}

\section{Discussion on Assumption \ref{emf}}\label{disemf1}
 
\subsection{Sufficient conditions for Assumption \ref{emf}}
\label{disem}
In this section, we first study the relationship between the extended MFCQ assumption and the MFCQ assumption. Let us state the MFCQ assumption \cite{mangasarian1967fritz} below.

\begin{assumption}[MFCQ \cite{mangasarian1967fritz}]\label{FCQ}
For every $\theta\in\Theta'$, where $\Theta'=\{\theta\in\Theta\mid \exists i\in[m], V_i^{\theta}(\rho)=0\}$, there exists a direction $s_{\theta}$ such that $\left\langle s_{\theta}, \nabla V_{i}^{\theta}(\rho) \right\rangle > 0$ for all $i \in \mathbf{B}_{0}(\theta):=\left\{i\in [m]\,|\, V_{i}^{\theta}(\rho)=0 \right\}$. 
\end{assumption}
Let us define
\begin{align*}
    \ell_{\theta}&:=\min_{i \in \B_{0}(\theta)}\left\{\left\langle \frac{s_{\theta}}{\|s_{\theta}\|}, \nabla V_{i}^{\theta}(\rho) \right\rangle \right\},\\
    \nu_{\theta}&:=\begin{cases}
\min_{i\in\left\{[m]\setminus \B_0(\theta)\right\}}\left\{V_i^{\theta}(\rho)\right\},\quad &\left\{[m]\setminus \B_0(\theta)\right\} \neq \emptyset, \\
\frac{1}{1-\gamma},\quad &{\text{otherwise}.}
\end{cases}
\end{align*}
Under the MFCQ assumption, for $\theta\in\Theta'$, both $\ell_{\theta}$ and $\nu_{\theta}$ are strictly positive. Now, we argue that under Assumption \ref{smoli}, we can establish a relationship between the MFCQ assumption and the extended MFCQ assumption.
\begin{proposition}\label{emfdis}
Let Assumptions \ref{smoli} and \ref{FCQ} hold. Set $\ell:=\inf_{\theta \in\Theta'}\left\{\frac{\ell_{\theta}}{2}\right\}$ and $\nu_{1}:=\inf_{\theta \in \Theta'}\left\{\frac{\nu_{\theta}}{3}\right\}$. If $\ell,\nu_{1} >0$, then Assumption \ref{emf}, namely, the extended MFCQ Assumption, holds.
\end{proposition}
\begin{proof}[Proof of Proposition \ref{emfdis}]
Under Assumption \ref{smoli}, we know that $V_i^\theta(\rho)$ is $L$-Lipschitz continuous and $M$-smooth, as shown in Proposition \ref{prosmo}. For each $\theta\in \Theta'$, consider $\theta_1\in \mathcal{R}(\theta) := \left\{\theta_1\mid \theta_1\in \Theta, \, \|\theta_1-\theta\|\le\min\{\frac{\ell_{\theta}}{2M}, \frac{\nu_{\theta}}{3L}\}\right\}$. For $i \in \B_0(\theta)$, we have
\begin{align*}
V_i^{\theta_1}(\rho) &\le V_i^{\theta}(\rho) + L\|\theta-\theta_1\| \le \frac{\nu_{\theta}}{3}.
\end{align*}
For $i \notin \B_0(\theta)$, we have
\begin{align*}
V_i^{\theta_1}(\rho) &\ge V_i^{\theta}(\rho) - L\|\theta-\theta_1\| \ge \frac{2\nu_{\theta}}{3}.
\end{align*}
Therefore, we can conclude that $\B_{\frac{\nu_{\theta}}{3}}(\theta_1) \subset \B_0(\theta)$ for $\theta_1\in \mathcal{R}(\theta)$. Next, we apply Assumption \ref{FCQ} on $\theta_1\in \mathcal{R}(\theta)$, we have for each $i \in \B_{\frac{\nu_{\theta}}{3}}(\theta_1)$
\begin{align*}
\left\langle \frac{s_{\theta}}{\|s_{\theta}\|}, \nabla V_{i}^{\theta_1}(\rho) \right\rangle 
=& \left\langle \frac{s_{\theta}}{\|s_{\theta}\|}, \nabla V_{i}^{\theta}(\rho) \right\rangle+\left\langle \frac{s_{\theta}}{\|s_{\theta}\|}, \nabla V_{i}^{\theta_1}(\rho) - \nabla V_{i}^{\theta}(\rho) \right\rangle\\
\ge &\ell_{\theta}-\left\|\nabla V_{i}^{\theta_1}(\rho) - \nabla V_{i}^{\theta}(\rho)\right\|\\
\ge &\ell_{\theta}-M\|\theta-\theta_1\|\\
\ge& \frac{\ell_{\theta}}{2}.
\end{align*} 
We further set $\nu_2$ as
\begin{align*}
\nu_2:=\inf\left\{V_i^\theta(\rho),i\in[m]\mid \theta \in \Theta\setminus \bigcup_{\theta\in\Theta'}\mathcal{R}(\theta)\right\}.
\end{align*}
Notice that $\nu_2 > 0$, we set $\nu_{emf}=\min\left\{\nu_1,\frac{\nu_2}{2}\right\}>0.$ Therefore, for any $\theta_2\in \Theta$, we have $\B_{\nu_{emf}}(\theta_2)\subset \bigcup_{\theta\in\Theta'} \B_{0}(\theta)$, since $\theta_2$ must be close to one of the $\theta$ in $\Theta'$ since $V_i^\theta(\rho)$ is a continuous function for $i\in[m]$. Consequently, we have
\begin{align*}
\left\langle \frac{s_{\theta}}{\|s_{\theta}\|}, \nabla V_{i}^{\theta}(\rho) \right\rangle \ge {\ell}.
\end{align*}
Therefore, Assumption \ref{emf} holds with such $\nu_{emf}$ and $\ell$.
\end{proof}
\begin{corollary}\label{equal}
    If $\Theta'$ is compact, Assumptions \ref{smoli} and \ref{FCQ} imply Assumption \ref{emf}.
\end{corollary}
\begin{proof}[Proof of Corollary \ref{equal}]
From Proposition \ref{emfdis}, it is sufficient to prove that $\ell,\nu_1$ defined in the above proposition are positive. We will prove this by contradiction.

Let us begin by proving $l > 0$ by contradiction. Suppose $\ell = 0$. This implies the existence of a series of points $\{\theta_i\}_{i=1}^{\infty}$ in $\Theta'$ such that
\[
\lim_{i\to \infty} \ell_{\theta_i} = 0.
\]
Using the definition of $\ell_{\theta_i}$, we have
\begin{align*}
   \lim_{i\to \infty} \sum_{j\in \B_0(\theta_i)}\left\|\nabla_\theta V_j^{\theta_i}(\rho)\right\|=0.
\end{align*}
Let $k_j := \sum_{i=1}^{\infty}\mathbf{1}_{j\in \B_0(\theta_i)}$ for $j\in[m]$. Since $\sum_{j=1}^{m}k_j = \sum_{i=1}^{\infty}\sum_{j=1}^{m}\mathbf{1}_{j\in \B_0(\theta_i)} = \sum_{i=1}^{\infty}\left|\B_0(\theta_i)\right|=\infty$, there exists a $j\in[m]$ such that $k_j=\infty$. We choose a subset of indices $\{i\}$ as $\{i_j\}$ such that $j\in \B_0(\theta_{i_j})$. Then, we have
\begin{align*}
   \lim_{i_j\to \infty}\|\nabla_\theta V_j^{\theta_{i_j}}(\rho)\|=0.
\end{align*}
Since $\Theta'$ is a compact set, there exists a $\theta_{lim}$ such that $\lim_{i_j\to\infty} \theta_{i_j}=\theta_{lim}$. For such $\theta_{lim}$, we have
\begin{align*}
    \left\|\nabla_\theta V_j^{\theta_{lim}}(\rho)\right\|=0 \quad \text{and} \quad V_j^{\theta_{lim}}(\rho)=0.
\end{align*}
However, this contradicts Assumption \ref{FCQ}. The same analysis applies for $\nu_1=0$, leading to a contradiction as well. 
\end{proof}

\textbf{Discussion } Notice that if there is only one constraint in the CMDP, this assumption is satisfied trivially. For the more general case with more than one constraint, in this section, we proved that the MFCQ assumption, which is weaker than Assumption 4.2, can ensure the satisfaction of Assumption \ref{emf} if our feasible policy parameterization set $ \{\theta \mid V_i^\theta(\rho) \ge 0\} $ is compact. This compactness can be achieved through certain policy parameterizations. For example, we can use direct parameterization or limit the policy parameterized set to $ \{\pi_\theta \mid \theta \in \mathbf{K}\} $, where $ \mathbf{K} $ is a compact set in $ \mathbf{R}^d $.



\subsection{Impact of Assumption \ref{emf}}
\label{disimpactemf}
The extended MFCQ assumption ensures that for every point $\theta$ lying on the boundary, there exists a trajectory that guides $\theta$ away from the boundary. Essentially, the extended MFCQ assumption prevents the algorithm from becoming trapped at the boundary, assuming a reasonable policy exists to guide the system back within the feasible region. When the CMDP structure lacks this property, safe exploration becomes more challenging to achieve. To illustrate this point, we provide the following theorem.
\begin{theorem}\label{without}
Let Assumptions \ref{smoli}, \ref{sl}, \ref{fn}, and \ref{bae} hold, and set $\eta<\nu_{emf}$, $H = \Tilde{\mathcal{O}}(\frac{1}{\eta})$ and $n = \mathcal{O}\left(\exp\frac{4}{\eta}\ln \frac{1}{\delta}\right)$ and $T = \mathcal{O}\left(\exp\frac{2}{\eta}\right)$. 
After $T$ iterations of the Algorithm \ref{alg:cap}, the following holds:
\begin{enumerate}
    \item $\PP\left(\forall t\in[T],\,\min_{i\in[m]} V_i^{\theta_t}(\rho)\ge \Omega\left(\exp{\frac{-1}{\eta(1-\gamma)}}\right)\right)
    \ge 1-mT\delta$.
    \item We can bound the regret of the objective function with a probability of at least $1-mT\delta$ as follows:
    \begin{align*}
        \frac{1}{T}\sum_{i=0}^{T-1}\left(V_{0}^{\pi^*}\left(\rho\right)-V_{0}^{\theta_t}\left(\rho\right)\right)& \le\mathcal{O}(\eta)+ \mathcal{O}\left(\sqrt{\varepsilon_{bias}}\exp{\frac{1}{\eta}}\right).
    \end{align*}
\end{enumerate}
\end{theorem}
We provide the proof of Theorem \ref{without} in the following section. This theorem illustrates that Algorithm \ref{alg:cap} requires high sample complexity to achieve safe exploration and does not guarantee the optimality of the iterates simultaneously if Assumption \ref{emf} is not satisfied. Specifically, without Assumption \ref{emf}, the LB-SGD iterations could be as close as  {$\mathcal{O}\left(\exp\frac{-1}{\eta}\right)$} to the boundary as shown in Theorem Property 1. 

To illustrate this, we provide an example demonstrating that without Assumption \ref{emf}, LB-SGD iterations might approach the boundary at a level of $\mathcal{O}(\eta^{2k+1})$ for any $k \in \mathcal{N}$. This closeness to the boundary leads to slower convergence due to smaller stepsizes and increased sample complexity as the iterates approach the boundary to maintain small bias and low variance of the estimators {$\hat V_i^\theta(\rho)$} and {$\hat \nabla V_i^\theta(\rho)$}. Meanwhile, we cannot guarantee the optimality of the iterates as it magnifies the transfer error {$\varepsilon_{bias}$} by {$\exp{\frac{1}{\eta}}$}.


\begin{exmp}
    We consider the problem as follows:
\begin{align*}
  \max_{x,y} \quad & -y\\
  \textrm{s.t.} \quad & y^{2k+1} + x \geq 0,\\
  & y^{2k+1} - 2x \geq 0,
\end{align*}
where $k\in \mathcal{N}$.
\end{exmp}
We can verify that the above example does not satisfy the MFCQ assumption since the constraint gradients at the point $(0,0)$ are opposite to each other. 
Next, we define the log barrier function as follows:
\begin{align*}
    B_\eta(x,y) = -y + \eta \log (y^{2k+1} + x) + \eta \log (y^{2k+1} - 2x).
\end{align*}
We compute that the optimal solution for the original problem is $(x^*, y^*) = (0, 0)$, and the optimal solution for the log barrier function is $(x^*_\eta, y^*_\eta) = (-2^{2k-1}(2k+1)^{2k+1}\eta^{2k+1}, (4k+2)\eta)$. When implementing gradient ascent on the log barrier function, starting from the point $(x_0, y_0) = (0, 1)$, the trajectory follows a curve from $(x_0, y_0)$ to the optimal point $(x^*_\eta, y^*_\eta)$. Meanwhile, $(x^*_\eta, y^*_\eta)$ is at a distance of $\mathcal{O}(\eta^{2k+1})$ from the boundary. Consequently, the iterates can approach the boundary within a distance of $\mathcal{O}(\eta^{2k+1})$.


\subsubsection{Proof of Theorem \ref{without}}\label{appwithout}
To prove Theorem \ref{without}, we follow a similar structure to the proof of Theorem \ref{main}. For safe exploration analysis (as outlined in Lemma \ref{small2}), we make use of Assumptions \ref{sl} and the boundedness of the value function $V_i^\theta(\rho)$ to establish a lower bound on the distance of the iterates from the boundary and the stepsize $\gamma_t$.

By employing a stochastic gradient ascent method with an appropriate stepsize $\gamma_t$, we ensure convergence to the stationary point of the log barrier function. Using Lemma \ref{gdm}, we can measure the optimality of this stationary point.

\begin{lemma}\label{small2}
Let Assumptions~\ref{smoli} and \ref{sl} hold, and we set $ n = \mathcal{O}\left(\exp\frac{4}{\eta}\ln \frac{1}{\delta}\right)$ and $H =\Tilde{\mathcal{O}}\left(\frac{1}{\eta}\right)$. Then, by running $T$ iterations of the LB-SGD algorithm, we obtain 
\begin{align*}
    &\PP\left\{\forall t\in[T],\,\min_{i\in[m]} V_i^{\theta_t}(\rho)\ge c_1,\,\gamma_t \ge C_1\text{ and }\left\|\Delta_t\right\|\ge\frac{\eta}{4}\right\}
    \ge 1-mT\delta,
\end{align*}
where { $c_1:=\nu_{s}^m (1-\gamma)^{m-1}\exp{\frac{-1}{\eta(1-\gamma)}}$}, { $C_1:=c_1^2\min\biggl\{\frac{3}{(\sqrt{6c_1 M}+4L)(L+m\eta L)},\frac{1}{c_1^2 M+20m\eta c_1 M+32m\eta L^2}\biggr\}$} and $\Delta_t:=\hgb-\gb $.
\end{lemma}
We first employ the sub-Gaussian tail bounds of the estimators $\hgf$ and $\hat V_i^{\theta_t}(\rho)$ to establish concentration bounds for $\left\|\hgb-\gb \right\|$ in Lemma \ref{gapdelt}. Additionally, as we apply the stochastic gradient ascent method with sufficient samples, $B_\eta^{\theta_t}(\rho)$ is non-decreasing. Combined with the boundedness of objective and constraint functions, we can establish a lower bound of $\mathcal{O}\left(\exp{\frac{-1}{\eta}}\right)$ for $V_i^{\theta_t}(\rho)$ with high probability.
\begin{proof}[Proof of Lemma \ref{small1}]
First, we prove the lower bound of the value functions $V_i^{\theta}(\rho),i\in\{0,\dots,m\}$. If we set 
\begin{align*}
     &{\sigma}^0(n) \le \min\left\{ \frac{1}{16\left(\sum_{i=1}^{m}\frac{ L}{\alpha_i(t)\hat \alpha_i(t)}\right)\sqrt{\ln \frac{2}{\delta}}}\right\}, {\sigma}^1(n) \le\min\left\{\frac{\eta}{16\left(1+\sum_{i=1}^{m}\frac{\eta}{\hat\alpha_i(t)}\right)\sqrt{\ln \frac{e^{\frac{1}{4}}}{\delta}}}\right\},\nonumber\\
     &b^0(H)\le \min\left\{ \frac{1}{16\left(\sum_{i=1}^{m}\frac{ L}{\alpha_i(t)\hat \alpha_i(t)}\right)}\right\},b^1(H)\le\min\left\{ \frac{\eta}{16\left(1+\sum_{i=1}^{m}\frac{\eta}{\hat\alpha_i(t)}\right)}\right\}.
\end{align*}
By Lemma \ref{gapdelt}, we have
\begin{align*}
    \PP(\Delta_t\le \frac{\eta}{4})\ge 1-\delta.
\end{align*}
Due to the choice of stepsize, $\PP(\gamma_t\le \frac{1}{M_t})\ge 1-\delta$, where $M_t$ is the local smoothness constant of the log barrier function $B_\eta^\theta(\rho)$. Then, we can bound $B_\eta^{\theta_{t+1}}(\rho)-B_\eta^{\theta_t}(\rho)$ with probability at least $1-\delta$ as follows:
\begin{align}
    &B_\eta^{\theta_{t+1}}(\rho)-B_\eta^{\theta_t}(\rho)\nonumber\\
    \ge& \gamma_t \left\langle\nabla_\theta B_\eta^{\theta_t}(\rho),\hat \nabla_\theta B_\eta^{\theta_t}(\rho)\right\rangle-\frac{M_t \gamma_t^2}{2}\left\|\hat \nabla_\theta B_\eta^{\theta_t}(\rho)\right\|^2\nonumber\\
    \ge& \gamma_t \left\langle\nabla_\theta B_\eta^{\theta_t}(\rho),\hat \nabla_\theta B_\eta^{\theta_t}(\rho)\right\rangle-\frac{\gamma_t}{2}\left\|\hat \nabla_\theta B_\eta^{\theta_t}(\rho)\right\|^2\nonumber\\
    = &\gamma_t \left\langle\nabla_\theta B_\eta^{\theta_t}(\rho),\left(\hat \nabla_\theta B_\eta^{\theta_t}(\rho)-\nabla_\theta B_\eta^{\theta_t}(\rho)\right)+\nabla_\theta B_\eta^{\theta_t}(\rho)\right\rangle-\frac{\gamma_t}{2}\left\|\left(\hat \nabla_\theta B_\eta^{\theta_t}(\rho)-\nabla_\theta B_\eta^{\theta_t}(\rho)\right)+\nabla_\theta B_\eta^{\theta_t}(\rho)\right\|^2\nonumber\\
    =&\frac{\gamma_t}{2}\left\| \nabla_\theta B_\eta^{\theta_t}(\rho)\right\|^2-\frac{\gamma_t}{2}\left\|\Delta_t\right\|^2.\label{ss1}
\end{align}
\normalsize
Before the break in algorithm \ref{alg:cap} line 5, we obtain that $\|\nabla_\theta B_\eta^{\theta_t}(\rho)\|\ge \frac{\eta}{4}$ since $\|\Delta_t\|\le \frac{\eta}{4}$. Therefore, 
$B_\eta^{\theta_{t+1}}(\rho)\ge B_\eta^{\theta_t}(\rho)$ which leads to $B_\eta^{\theta_{t}}(\rho)\ge B_\eta^{\theta_0}(\rho)$. Then,
\begin{align*}
   &V_0^{\theta_t}(\rho)+\eta\sum_{i\in[m]}\log V_i^{\theta_t}(\rho) \ge V_0^{\theta_0}(\rho)+\eta\sum_{i\in[m]}\log V_i^{\theta_0}(\rho),\\
   &\log V_j^{\theta_t} \ge \frac{V_0^{\theta_0}(\rho)-V_0^{\theta_t}(\rho)}{\eta}+\sum_{i\in[m]}\log V_i^{\theta_0}(\rho)- \sum_{\substack{i\in[m]\\i\neq j}} \log V_i^{\theta_t},\\
   &\log V_j^{\theta_t} \ge \frac{-1}{(1-\gamma)}+m\log \nu_{s}+ (m-1)\log (1-\gamma),
\end{align*}
where the last inequality comes from the boundness of the value functions $V_i^{\theta}(\rho),i\in\{0,\dots,m\}$. Therefore,
\begin{align}
    \min_{i\in[m]} V_i^{\theta_t}(\rho)\ge c_1, \text{$ \forall t\in\{0,\dots,T\}$,}\label{lowerbound}
\end{align}
where $c_1:=\nu_{s}^{m} (1-\gamma)^{m-1}\exp{\frac{-1}{\eta(1-\gamma)}} = \mathcal{O}\left(\exp{\frac{-1}{\eta}}\right)$. For each $i\in[m]$, if $\sigma^0(n) \leq \frac{\alpha_i(t)}{8\sqrt{\ln\frac{2}{\delta}}}$ and $b^0(H) \leq \frac{\alpha_i(t)}{8}$, we have $\mathbb{P}\left(\frac{3\alpha_i(t)}{4} \leq \hat \alpha_i(t)\right) \geq 1-\delta$ using the sub-Gaussian bound in Proposition \ref{prosmo}. Therefore, we need to bound the variances and biases as follows to make sure $\PP(\Delta_t\le \frac{\eta}{4})\ge 1-\delta$.
\begin{align*}
     &{\sigma}^0(n) \le \min\left\{ \frac{3}{64\left(\sum_{i=1}^{m}\frac{ L}{(\alpha_i(t))^2}\right)\sqrt{\ln \frac{2}{\delta}}},\frac{\alpha_i(t)}{8\sqrt{\ln\frac{2}{\delta}}}\right\},{\sigma}^1(n) \le\min\left\{\frac{3\eta}{64\left(1+\sum_{i=1}^{m}\frac{\eta}{\alpha_i(t)}\right)\sqrt{\ln \frac{e^{\frac{1}{4}}}{\delta}}}\right\},\nonumber\\
     &b^0(H)\le \min\left\{ \frac{3}{64\left(\sum_{i=1}^{m}\frac{ L}{(\alpha_i(t))^2}\right)},\frac{\alpha_i(t)}{8}\right\},b^1(H)\le\min\left\{ \frac{3\eta}{64\left(1+\sum_{i=1}^{m}\frac{\eta}{\alpha_i(t)}\right)}\right\}.
\end{align*}
\normalsize
By the lower bound in \eqref{lowerbound} and the Proposition \ref{prosmo}, the number of trajectories $n$ and the truncated horizon $H$ need to be set as follows:
\begin{align*}
    H = \Tilde{\mathcal{O}}\left(\frac{1}{\eta}\right), n = \mathcal{O}\left(\exp\frac{4}{\eta}\ln \frac{1}{\delta}\right).
\end{align*}
Meanwhile, we can further lower bound $\gamma_t$ which is 
\begin{align*}
    \gamma_t:=\min&\Biggl\{\min_{i\in[m]}\left\{\frac{\underline{\alpha}_i(t)}{\sqrt{M_i \underline{\alpha}_i(t)}+2|\overline{\beta}_i(t)|}\right\}\frac{1}{\|\hgb\|},\frac{1}{M +\sum_{i=1}^{m}\frac{10M\eta }{\underline{\alpha}_t^i}+8\eta \sum_{i=1}^{m}\frac{\left(\overline{\beta}_t^i\right)^2}{\left(\underline{\alpha}_t^i\right)^2}}\Biggr\}.
\end{align*}
Since $\sigma^0(n) \leq \frac{\alpha_i(t)}{8\sqrt{\ln\frac{2}{\delta}}}$ and $b^0(H) \leq \frac{\alpha_i(t)}{8}$, we have $\mathbb{P}\left(\frac{\alpha_i(t)}{2} \leq \underline \alpha_i(t)\le \frac{3}{2\alpha_i(t)}\right) \geq 1-\delta$ using the sub-Gaussian bound in Proposition \ref{prosmo}. Together with \eqref{lowerbound}, we have
 \begin{align*}
    \PP\left(\gamma_t\ge C_1\right)\ge 1-\delta,
\end{align*}
where $C_1$ is defined as
\begin{align*}
    C_1:=c_1^2\min&\biggl\{\frac{3}{(\sqrt{6c_1 M}+4L)(L+m\eta L)},\frac{1}{c_1^2 M+20m\eta c_1 M+32m\eta L^2}\biggr\}
\end{align*}
which is at the level of $\mathcal{O}\left(\exp\frac{-2}{\eta}\right)$.
\end{proof}
\begin{proof}[Proof of Theorem \ref{without}] We set the values for $n$, $H$, and $\eta$ to satisfy the conditions outlined in Lemma \ref{small2}. Due to our choice of stepsize, we have $\PP\left(\gamma_t \leq \frac{1}{M_{t}}\right) \geq 1-\delta$, where $M_{t}$ represents the local smoothness constant of the log barrier function $B_\eta^\theta(\rho)$. With this, we can bound $B_\eta^{\theta_{t+1}}(\rho) - B_\eta^{\theta_t}(\rho)$ with a probability of at least $1-\delta$ as
\begin{align}
    B_\eta^{\theta_{t+1}}(\rho) - B_\eta^{\theta_t}(\rho)
    \geq& \gamma_t \left\langle \nabla_\theta B_\eta^{\theta_t}(\rho),\hat \nabla_\theta B_\eta^{\theta_t}(\rho)\right\rangle - \frac{M_{t} \gamma_t^2}{2}\left\|\hat \nabla_\theta B_\eta^{\theta_t}(\rho)\right\|^2\nonumber\\
    \geq& \gamma_t \left\langle \nabla_\theta B_\eta^{\theta_t}(\rho),\hat \nabla_\theta B_\eta^{\theta_t}(\rho)\right\rangle - \frac{\gamma_t}{2}\left\|\hat \nabla_\theta B_\eta^{\theta_t}(\rho)\right\|^2\nonumber\\
    =& \gamma_t \left\langle \nabla_\theta B_\eta^{\theta_t}(\rho),\Delta_t + \nabla_\theta B_\eta^{\theta_t}(\rho)\right\rangle - \frac{\gamma_t}{2}\left\|\Delta_t + \nabla_\theta B_\eta^{\theta_t}(\rho)\right\|^2\nonumber\\
    =&\frac{\gamma_t}{2}\left\| \nabla_\theta B_\eta^{\theta_t}(\rho)\right\|^2-\frac{\gamma_t}{2}\left\|\Delta_t\right\|^2.\label{ss}
\end{align}
We divide the analysis into two cases based on the \textbf{if condition} in algorithm \ref{alg:cap} line 5.

\textbf{Case 1:} If $\|\hat \nabla_\theta B_\eta^{\theta_t}(\rho)\|\ge \frac{\eta}{2}$, then $\|\nabla_\theta B_\eta^{\theta_t}(\rho)\|\ge \frac{\eta}{4}$ because $\|\Delta_t\|\le \frac{\eta}{4}$ by Proposition \ref{small}. We can bound \eqref{ss} as
\begin{align}
B_\eta^{\theta_{t+1}}(\rho) - B_\eta^{\theta_t}(\rho) &\ge \frac{C_1\eta}{8}\left\|\nabla_\theta B_\eta^{\theta_t}(\rho)\right\| - \frac{C_1\eta^2}{32},\label{case1}
\end{align}
where we plug in $\gamma_t\ge C_1$ in the last step. Summing inequality \eqref{case1} from $t=0$ to $t=T-1$, we have
\begin{align*}
    B_\eta^{\theta_{T}}(\rho) - B_\eta^{\theta_0}(\rho) + \frac{ C_1\eta^2 T}{32}&\ge \sum_{t=0}^{T-1}\frac{C_1\eta}{8}\left\|\nabla_\theta B_\eta^{\theta_t}(\rho)\right\|\\
    \frac{8( B_\eta^{\theta_{T}}(\rho) - B_\eta^{\theta_0}(\rho))}{C_1 \eta T} + \frac{ \eta }{4}&\ge \frac{1}{T}\sum_{t=0}^{T-1}\left\|\nabla_\theta B_\eta^{\theta_t}(\rho)\right\|.
\end{align*}
Since the value functions $V_i^\theta(\rho)$ are upper bounded by $\frac{1}{1-\gamma}$, we can further bound the above inequality as
\begin{align*}
    \frac{1}{T}\sum_{t=0}^{T-1}\left\|\nabla_\theta B_\eta^{\theta_t}(\rho)\right\|\le 
    \frac{8\left( \frac{1}{1-\gamma}-m\eta \log (1-\gamma)- B_\eta^{\theta_0}(\rho))\right)}{C_1 \eta T} + \frac{ \eta }{4}
\end{align*}
By setting $T=\mathcal{O}\left(\frac{1}{C_1\eta^2}\right)=\mathcal{O}\left(\exp\frac{2}{\eta}\right)$, we have
\begin{align}
    \frac{1}{T}\sum_{t=0}^{T-1}\left\|\nabla_\theta B_\eta^{\theta_t}(\rho)\right\|\le 
   \mathcal{O}(\eta)\label{a1}
\end{align}
\textbf{Case 2:} If $ \|\hat \nabla_\theta B_\eta^{\theta_t}(\rho)\|\le \frac{\eta}{2}$, we have 
\begin{align}
    \| \nabla_\theta B_\eta^{\theta_t}(\rho)\|\le  \|\hat \nabla_\theta B_\eta^{\theta_t}(\rho)\|+\|\Delta_t\|\le \frac{3\eta}{4}.\label{a2}
\end{align}
Applying Lemma \ref{gdm} on \eqref{a1} and \eqref{a2}, we have
\begin{align*}
 \frac{1}{T}\sum_{t=0}^{T-1}\left(V_{0}^{\pi^*}\left(\rho\right)-V_{0}^{\theta_t}\left(\rho\right)\right)
\le& m\eta+ \frac{1}{T}\sum_{t=0}^{T-1}\left(\frac{\sqrt{\varepsilon_{bias}}}{1-\gamma}\left(1+\sum_{i=1}^{m}\frac{\eta}{ V_{i}^{\theta_t}\left(\rho\right)}\right)\right) +\mathcal{O}(\eta)\nonumber\\
=&\mathcal{O}(\eta)+\mathcal{O}\left(\sqrt{\varepsilon_{bias}}\exp{\frac{1}{\eta}}\right),
\end{align*}
where we use $V_i^\theta(\rho)\ge c_1=\mathcal{O}(\exp\frac{-1}{\eta})$ in the last step. Therefore, we can conclude the following: after $T = \mathcal{O}\left(\exp\frac{2}{\varepsilon}\right)$ iterations of the LB-SGD algorithm with $\eta = \varepsilon$, we have
\begin{align*}
   \frac{1}{T}\sum_{t=0}^{T-1}\left(V_{0}^{\pi^*}\left(\rho\right)-V_{0}^{\theta_t}\left(\rho\right)\right) \le\mathcal{O}(\varepsilon)+\mathcal{O}\left(\sqrt{\varepsilon_{bias}}\exp{\frac{1}{\varepsilon}}\right),
\end{align*}
while safe exploration is ensured with a probability of at least $1-mT\delta$.
\end{proof}
\section{Policy parameterization}
\label{policyop}
In this section, we delve deeper into widely accepted assumptions and derive conditions under which they may or may not hold. Specifically, we investigate the relationship between the (relaxed) Fisher non-degenerate assumption and the bounded transfer error assumption, especially concerning commonly used tabular policy parameterizations. These parameterizations include softmax, log-linear, and neural softmax policies. These policies are  defined as follows:
    $$\pi_\theta(a|s)=\frac{\exp{f_\theta(s,a)}}{\sum_{a'\in \A} \exp{f_\theta(s,a')}}.$$
\begin{enumerate}
    \item For softmax policy, $f_\theta(s,a)=\theta(s,a)$.
    \item For log-linear policy, $f_\theta(s,a)=\theta^T\cdot \phi(s,a)$, with $\theta\in \R^d$ and $\phi(s,a)\in \R^d$.
    \item For neural softmax policy, $f_\theta(s,a)$ is parameterized using neural networks. 
\end{enumerate}
We first introduce two critical concepts related to policy parameterization: the $\varepsilon$-deterministic policy and richness in the policy parameterization.

\begin{definition}[\(\varepsilon\)-deterministic Policy]\label{deter}
    We define a policy, \(\pi_\theta\), as an \(\varepsilon\)-deterministic policy if $\pi_\theta\in \Pi_\varepsilon:=\{\pi_\theta\,|\,\text{for every state $s$, there exists } a_{i_s}\in\mathcal{A} \text{ such that }\pi_\theta(a_{i_s}|s)\ge 1-\varepsilon\}$. As $\varepsilon$ approaches zero, the policy is said to approach a deterministic policy.
\end{definition}

\begin{definition}[Richness of Policy Parameterization]
    Define \(\Pi\) as the closure of all stochastically parameterized policies, denoted by \(\textbf{Cl}\{\pi_\theta\mid\theta\in \mathbb{R}^{d}\}\). If we have another policy parameterization \(\Pi'\), we say \(\Pi'\) is a richer parameterization compared to \(\Pi\) if \(\Pi'\subsetneq\Pi\).
\end{definition}
In the following section, we examine the commonly employed assumptions of (relaxed) Fisher non-degeneracy and bounded transfer error with tabular policy parameterizations as defined above. In Section \ref{disfn}, we demonstrate that softmax parameterization cannot satisfy Fisher non-degeneracy, and softmax, log-linear, and neural softmax parameterizations fail to meet relaxed Fisher non-degeneracy as the policy approaches a deterministic policy. This reveals the relationship between the relaxed Fisher non-degenerate assumption and the exploration rate of the policy. In Section \ref{disbae}, we prove that the transfer error bound can be reduced by increasing the richness of the policy set for log-linear and neural softmax policy parameterizations, indicating the relationship between the transfer error bound assumption and the richness of the policy parameterization.

\subsection{Discussion on Assumption \ref{fn}}\label{disfn}
In this section, we divide the analysis into two parts. In Section \ref{softf}, we prove that the softmax parameterization cannot satisfy the Fisher non-degeneracy assumption, and fails to satisfy relaxed Fisher non-degeneracy as the policy approaches determinism. Then, in Section \ref{loglineaar}, we demonstrate that log-linear and neural softmax parameterizations also fail to satisfy relaxed Fisher non-degeneracy when the policy approaches a deterministic state. This section reveals the relationship between the relaxed Fisher non-degeneracy assumption and the exploration rate of the policy, highlighting the limitations of algorithms that depend on satisfying this relaxed assumption.

\subsubsection{Softmax parameterization}
\label{softf}

 The authors of \cite{ding2022global} claim that the softmax parameterization fails to satisfy Fisher non-degeneracy when the policy approaches a deterministic policy. In this section, we prove a stronger version of this claim,
 \begin{proposition}\label{softmaxcannot}
 \begin{itemize}
     \item Softmax parameterization does not satisfy Fisher non-degeneracy.
     \item Softmax fails to satisfy relaxed Fisher non-degeneracy as the policy approaches determinism.
 \end{itemize}

 \end{proposition}
  \begin{proof}
The first property may seem obvious by \cite[Appendix]{metelli2022policy}, as the softmax parameterization requires {\small $|S||A|$} parameters, which can result in rank deficiency in the Fisher information matrix due to the redundancy of {\small $|S|$} parameters. For completeness, we still provide a formal proof of this statement below.

Note that the Fisher information matrix is computed as:
\[ F^\theta(\rho)=\mathbb{E}_{(s,a)\sim d_\rho^{\theta}}[\nabla_\theta \log\pi_\theta(a|s)\left(\nabla_\theta\log\pi_\theta(a|s)\right)^T] \]
Therefore, the image of the Fisher information matrix is the span of the linear space \(\{\nabla_\theta\log\pi_\theta(a|s),\forall a\in\A, s\in\St\}\) if the state-action occupancy measure \(d_\rho^{\theta}\) is non-zero for every state-action pair \((s,a)\).

For softmax parameterization, \(\nabla_{\theta_{s'}}\log\pi_\theta(a|s)\) can be computed as:
\[ \nabla_{\theta_{s'}} \log \pi_\theta(a|s)= \mathbf{1}_{s'=s}\left(e_{a}-\pi(\cdot| s)\right), \]
where $\theta_{s'}:=\{\theta(s',a_1),\dots,\theta(s',a_{|\A|})\}$. Here, \(e_{a}\in\R^{|\A|}\) is an elementary vector, with zeros everywhere except in the \(a\)th position, and \(\pi(\cdot|s):=\left(\pi_\theta(a_1|s),\dots,\pi_\theta(a_{|\A|}|s)\right)^{T}\). Therefore, $\{\nabla_{\theta_s'}\log\pi_\theta(a|s),\allowbreak\forall a\in\A\}=\{\mathbf{0}\}$ if \(s'\neq s\) and
\begin{align*}
    \left\{\nabla_{\theta_s}\log\pi_\theta(a|s),\forall a\in\A\right\}
    =\{e_{a}-(\pi(\cdot|s), \forall a\in\A\}
    = \{e_{a_1}-(\pi(\cdot|s), e_{a_1}-e_{a_i},\forall i\in[|\A|]\},
\end{align*}
where the first equality is written by definition and second equality is computed by $(e_{a}-(\pi(\cdot|s))-(e_{a'}-(\pi(\cdot|s))$. The rank of \(\left\{\nabla_{\theta_s}\log\pi_\theta(a|s),\forall a\in\A\right\}\) is \(|\A|-1\) since the rank of $\{e_{a_1}-e_{a_i},\forall i\in[|\A|]\}$ is \(|\A|-1\). Additionally, one vector \(\mathbf{1}\in\R^{|\A|}\) is orthogonal to \(\left\{\nabla_{\theta_s}\log\pi_\theta(a|s),\forall a\in\A\right\}\) due to:
\[ \mathrm{1}^T \left(e_{a}-\pi(\cdot| s)\right)= 1 -\sum_{a\in \A}\pi(a|s)=0.\]
In conclusion, \(F^\theta(\rho)\) has rank  at most \((|\A|-1)|\St|\). This indicates that the Fisher information matrix is not a full-rank matrix, and therefore unable to satisfy the Fisher non-degeneracy.

To prove the second part of the proposition, we show that the trace of the Fisher information matrix, denoted as $\text{tr}(F^\theta(\rho))$, approaches zero as the policy approaches determinism. Consequently, the smallest non-zero eigenvalues of the Fisher information matrix also converge to zero, thereby fails to satisfy the relaxed Fisher non-degeneracy assumption.

We compute $\text{tr}(F^\theta(\rho))$ as follows:
\begin{align}
    \text{tr}(F^\theta(\rho)) = \sum_{s} d_\rho^\theta(s)\sum_{a\in\mathcal{A}}\pi(a| s)\left(\sum_{a'\in\mathcal{A},a'\neq a}\pi(a'|s)^2 + (1-\pi(a|s))^2\right).\label{soft1}
\end{align}
Given a policy $\pi_\theta\in\Pi_\varepsilon$, for any state $s\in\mathcal{S}$, there exists an action $a_{i_s}\in\mathcal{A}$ such that $\pi_\theta(a_{i_s}|s)\ge 1-\varepsilon$. Therefore, we have
\begin{align*}
    \sum_{a'\in\mathcal{A},a'\neq a_{i_s}}\pi(a'|s)^2 + (1-\pi(a_{i_s}|s))^2\le \left(\sum_{a'\in\mathcal{A},a'\neq a_{i_s}}\pi(a'|s)+ 1-\pi(a_{i_s}|s)\right)^2\le 4\varepsilon^2.
\end{align*}
Plugging the above inequality into inequality \eqref{soft1}, we have
\begin{align*}
    \text{tr}(F^\theta(\rho)) &= \sum_{s} d_\rho^\theta(s)\sum_{a\in\mathcal{A}}\pi(a| s)\left(\sum_{a'\in\mathcal{A},a'\neq a}\pi(a'|s)^2 + (1-\pi(a|s))^2\right)\\
    &\le \sum_{s} d_\rho^\theta(s)\sum_{a\in\mathcal{A},a\neq a_{i_s}}\pi(a| s)\left(\sum_{a'\in\mathcal{A},a'\neq a}\pi(a'|s)^2 + (1-\pi(a|s))^2\right)+4\varepsilon^2\sum_{s} d_\rho^\theta(s)\pi(a_{i_s}|s)\\
    &\le \sum_{s} d_\rho^\theta(s)\sum_{a\in\mathcal{A},a\neq a_{i_s}}\pi(a| s)\left(\sum_{a'\in\mathcal{A},a'\neq a}\pi(a'|s)^2 + (1-\pi(a|s))^2\right)+4\varepsilon^2\\
    &\le \sum_{s} d_\rho^\theta(s)\sum_{a\in\mathcal{A},a\neq a_{i_s}}\pi(a| s)\left(\sum_{a'\in\mathcal{A},a'\neq a}\pi(a'|s) + 1-\pi(a|s)\right)^2+4\varepsilon^2\\
    &\le 4\sum_{s} d_\rho^\theta(s)\sum_{a\in\mathcal{A},a\neq a_{i_s}}\pi(a| s)+4\varepsilon^2\\
    &\le 4\varepsilon+4\varepsilon^2.
\end{align*}
Therefore, when a policy $\pi_\theta$ approaches a deterministic policy, the trace of Fisher information matrix approaches 0.
\end{proof}
\subsubsection{Log-linear and neural softmax parameterizations}\label{loglineaar}
It is known that for MDPs, the optimal policy can be deterministic. For CMDPs, the optimal policy can be deterministic if none of the constraints are active. Furthermore, during the implementation of policy gradient-based algorithms, the algorithm can be trapped at a stationary point, which can be deterministic. While log-linear and neural softmax parameterizations are commonly employed for discrete state and action spaces, this section reveals their failure to exhibit relaxed Fisher non-degeneracy when approaching a deterministic policy.
\begin{proposition}\label{fnlog}
Let $\|\nabla_\theta f_\theta(s,a)\| \le M$ hold for all $s\in\St$ and $a\in\A$, Log-linear and neural softmax parameterizations fail to meet Assumption \ref{fn} as the policy approaches a deterministic policy.
\end{proposition}
\begin{proof}
For a policy $\pi_\theta\in\Pi_\varepsilon$, which is $\varepsilon$-close to a deterministic policy, we first compute $\nabla_\theta \log \pi_\theta(a|s)$ as follows:
\begin{align*}
    \nabla_\theta \log \pi_\theta(a|s) 
    =& \nabla f_\theta(s,a)-\sum_{a'\in A}\nabla f_\theta(s,a') \pi_\theta(a'|s) 
    =\sum_{a'\in A,a'\neq a}\left(\nabla f_\theta(s,a)-\nabla f_\theta(s,a')\right) \pi_\theta(a'|s).
\end{align*}
To satisfy Assumption \ref{smoli} for log-linear and neural softmax parameterizations, it is commonly assumed that $\|\nabla f_\theta(s,a)\| \le M$ for all $a\in \mathcal{A}$ and $s\in \mathcal{S}$. Under this condition, we can bound $\nabla_\theta \log \pi_\theta(a_{i}|s)$ into two cases. For $i=i_s$, we know that $\sum_{a'\in A,a'\neq a_{i_s}}\pi_\theta(a'|s) \le \varepsilon$, therefore
\begin{align*}
   &\|\nabla_\theta \log \pi_\theta(a_{i_s}|s)\|
   = \|\sum_{a'\in A,a'\neq a_{i_s}}\left(\nabla f_\theta(a_i|s)-\nabla f_\theta(a'|s)\right) \pi(a'|s)\| 
   \le 2M \varepsilon.
\end{align*}
For $i\neq i_s$, we know that $\sum_{a'\in A,a'\neq a_{i}}\pi_\theta(a'|s) \le 1$, therefore
\begin{align*}
&\|\nabla_\theta \log \pi_\theta(a_{i}|s)\|
= \|\sum_{a'\in A,a'\neq a_i}\left(\nabla f_\theta(a_i|s)-\nabla f_\theta(a'|s)\right) \pi(a'|s)\| 
\le 2M.
\end{align*}
Using the above two inequalities, we can upper bound the norm of the Fisher information matrix for the policy $\pi_\theta$ as
\begin{align}
   \left\| F^\theta(\rho)\right\|
   =&\|\sum_{s\in \St} d_\rho^\theta(s)\sum_{i=1}^{|\A|} \pi_\theta(a_{i_s}|s)\nabla_\theta \log\pi_\theta(a_{i_s}|s)\left(\nabla_\theta\log\pi_\theta(a_{i_s}|s)\right)^T\|\nonumber\\
    \le&\|\sum_{s\in \St} d_\rho^\theta(s)\pi_\theta(a_{i_s}|s)\nabla_\theta\log\pi_\theta(a_{i_s}|s)\left(\nabla_\theta\log\pi_\theta(a_{i_s}|s)\right)^T\|\nonumber\\
    &+ \|\sum_{s\in \St} d_\rho^\theta(s)\sum_{a\neq a_{i_s},a\in \A} \pi(a|s)\nabla \log\pi_\theta(a|s)\left(\nabla \log\pi_\theta(a|s)\right)^T\|\nonumber\\
    \le& 4M^2 \varepsilon^2\|\sum_{s\in \St} d_\rho^\theta(s)\|+ 4M^2\|\sum_{s\in \St} d_\rho^\theta(s)\sum_{a_i\neq a_{i_s},a_i\in \A} \pi_\theta(a_{i}|s_i)\|\nonumber\\
    \le& 4\varepsilon^2M^2+4\varepsilon M^2.\nonumber
\end{align}
Therefore, when a policy $\pi_\theta$ approaches a deterministic policy, the norm of Fisher information matrix approaches 0. Then, we cannot find a positive constant $\mu_F$ such that the smallest non-zero eigenvalue of {$F^\theta(\rho)$} is lower bounded by {$\mu_F$}.
\end{proof}
\textbf{Discussion} For continuous action space, Assumption \ref{fn} is satisfied by Gaussian policies $\pi_\theta(\cdot|s) = \mathcal{N}(\mu_\theta(s), \Sigma)$ when the parameterized mean $\mu_\theta(s)$ has a full row rank Jacobian and the covariance matrix $\Sigma$ is fixed. Cauchy policies also satisfy these conditions.

   For discrete action space, we proved that softmax, log-linear and neural softmax parameterizations fail to satisfy Assumption \ref{fn} as the policy becomes more deterministic. Building upon this fact, to ensure softmax, log-linear and neural softmax parameterizations satisfy relaxed Fisher non-degeneracy, we can constrain the parameter space to a compact set, preventing the policy from approaching determinism.
\subsection{Discussion on Assumption \ref{bae}}\label{disbae}
Inspired by \cite{wang2019neural}, where the authors demonstrated that employing a rich two-layer neural-network parameterization can yield small \(\varepsilon_{bias}\) values, we generalize their result to demonstrate that increasing the richness of the policy set can lead to a reduction in the transfer error \(\varepsilon_{bias}\) for the log-linear and neural softmax policy parameterizations.

We consider the log-linear and neural softmax policy parameterization given by
$$\pi_\theta(a|s)=\frac{\exp{f_\theta(s,a)}}{\sum_{a'\in \A} \exp{f_\theta(s,a')}}.$$
\begin{proposition}
For log-linear and neural softmax policy parameterizations, increasing the dimension of the set \(\{f'(s,a),s\in \mathcal{S},a\in\mathcal{A}\}\) such that \(\{f(s,a),s\in \mathcal{S},a\in\mathcal{A}\}\subsetneq\{f'(s,a),s\in \mathcal{S},a\in\mathcal{A}\}\) results in a richer parameterization. This richer parameterization leads to a decrease in the transfer error \(\varepsilon_{bias}\).
\end{proposition}
\begin{proof}
We first upper bound the transfer error as shown in \cite[page 29]{agarwal2021theory}:
\[
\begin{aligned}
    &L(\mu_i^*,\theta,d_\rho^{\pi^*}) \le \max_{s\in\St}\frac{\sum_{a\in\A}d_\rho^{\pi^*}(s,a)}{(1-\gamma)\rho(s)}L(\mu_i^*,\theta,d_\rho^{\theta}),
\end{aligned}
\]
where \(L(\mu_i^*,\theta,d_\rho^{\theta}) := \mathbb{E}_{(s,a)\sim d_\rho^{\theta}}[({A}_i^\theta(s,a)-(1-\gamma){\mu_i^*}^T \nabla_\theta\log\pi_\theta(a|s))^2]\) and \(\mu_i^* := \left(F^\theta(\rho)\right)^{-1} \nabla_\theta V_i^\theta(\rho)\). In the following, we demonstrate that \(L(\mu_i^*,\theta,d_\rho^{\theta})\) can be reduced due to the richer parametrization.

For every policy $\pi_\theta \in \Pi_{f}$, we set a vector-valued function $\mathbb{A}: \Pi_{f} \to \mathbb{R}^{|\St||\A|}$ as
\begin{equation*}
    \mathbb{A}_i^{\pi_\theta}:=\begin{bmatrix}
\sqrt{d_\rho^{\pi}(s_1,a_1)}{A}_i^{\pi}(s_1,a_1) \\
\vdots \\
\sqrt{d_\rho^{\pi}(s_{|\St|},a_{|\A|})}{A}_i^{\pi}(s_{|\St|},a_{|\A|})
\end{bmatrix}\nonumber
\end{equation*}
and set ${\mathbb{B}}$ as
\begin{equation*}
    \mathbb{B}_{\pi_\theta}:=\begin{bmatrix}
\sqrt{d_\rho^{\pi}(s_1,a_1)}\nabla_\theta \log\pi_\theta(a_{1}|s_1)\\
\vdots \\
\sqrt{d_\rho^{\pi}(s_{|\St|},a_{|\A|})}\nabla_\theta \log\pi_\theta(a_{|\A|}|s_{|\St|})
\end{bmatrix}^T.\nonumber
\end{equation*}
Notice that 
\begin{align*}
    \nabla_\theta \log \pi_\theta(a|s)=\sum_{\substack{a'\in \A, \,a'\neq a}}\left(\nabla f_\theta(s,a)-\nabla f_\theta(s,a')\right) \pi_\theta(a'|s).
\end{align*}
Then the column space of $\mathbb{B}_{\pi_\theta}$ is the span of $\{\nabla_\theta f(s,a)\}_{\substack{s\in\St,\, a\in\A}}$. Using the above notations, we write Fisher information matrix as 
\begin{align*}
    F^{\pi_\theta}(\rho)&=\mathbb{E}_{(s,a)\sim d_\rho^{\pi}}[\nabla_\theta \log\pi_\theta(a|s)\left(\nabla_\theta \log\pi_\theta(a|s)\right)^T]= \mathbb{B}_{\pi_\theta} \mathbb{B}_{\pi_\theta}^T,\nonumber
\end{align*}
and the gradient of value function as
\begin{align*}
    \nabla_\theta V_i^\theta(\rho)&=\frac{1}{1-\gamma}\mathbb{E}_{(s,a)\sim d_\rho^\pi}\left[\nabla_\theta\log\pi_\theta(a|s){A}_i^\theta(s,a)\right]=\frac{1}{1-\gamma}\mathbb{B}_{\pi_\theta} \mathbb{A}_i^{\pi_\theta}.\nonumber
\end{align*}
and $\mu_i^*= F^{\pi_\theta}(\rho)^{-1} \nabla_\theta V_i^\theta(\rho)=\frac{1}{1-\gamma}(\mathbb{B}_\pi \mathbb{B}^T_\pi)^{-1}\mathbb{B}_\pi \mathbb{A}_i^\pi$. Therefore, we have
\begin{align*}
L(\mu_i^*,\pi,d_\rho^{\pi_\theta})
=&\mathbb{E}_{(s,a)\sim d_\rho^{\pi}}\left[\left({A}_i^{\pi_\theta}(s,a)-(1-\gamma){\mu_i^*}^T \nabla_\theta\log\pi_\theta(a|s)\right)^2\right]\\
=&\sum_{s,a} \Bigl(\sqrt{d_\rho^{\pi}(s,a)}\left(\mathbb{B}_{\pi_\theta} \mathbb{B}_{\pi_\theta}^T\right)^{\dagger}\mathbb{B}_{\pi_\theta} \mathbb{A}_i^{\pi_\theta}\nabla_\theta \log\pi_\theta(a|s)-\sqrt{d_\rho^{\pi}(s,a)}{A}_i^{\pi_\theta}(s,a)\Bigr)^2\\
=&\left\|\mathbb{A}_i^{\pi_\theta}-\mathbb{B}_{\pi_\theta}^T(\mathbb{B}_{\pi_\theta} \mathbb{B}^T_{\pi_\theta})^{\dagger}\mathbb{B}_{\pi_\theta} \mathbb{A}_i^{\pi_\theta}\right\|_2^2\\
=&\left\|\mathbb{A}_i^{\pi_\theta}-\mathbf{P}_{\mathbb{B}_{\pi_\theta}}\mathbb{A}_i^{\pi_\theta}\right\|_2^2,
\end{align*}
where $\mathbf{P}_{\mathbb{B}_{\pi_\theta}}$ is the orthogonal projection onto the row space of $\mathbb{B}_{\pi_\theta}\in \mathbb{R}^{d\times|\mathcal{S}||\mathcal{A}|}$. If we increase the dimension of $\{f'(s,a),s\in \mathcal{S},a\in\mathcal{A}\}$ such that $\{f(s,a),s\in \mathcal{S},a\in\mathcal{A}\}\subsetneq\{f'(s,a),s\in \mathcal{S},a\in\mathcal{A}\}$, it results in a more richer parameterization and a decrease in $L(\mu_i^*,\pi,d_\rho^{\pi})$ since the column rank of $\mathbb{B}_{\pi_\theta}$ has increased.
\end{proof}
\section{Proofs}
\label{appproof}
\subsection{Proof of Proposition~\ref{prosmo}}\label{proofsml}
To set up sub-Gaussian bounds for the gradient estimates in the RL case, we require the following lemma.
\begin{lemma}\label{veber}
\textbf{Vector Bernstein Inequality \cite[Lemma 18]{kohler2017sub}: } Let $x_i\in\R^d$ be independent vector-valued random variables for $i\in[n]$. If there exist constants $B,\sigma\ge 0$ such that $\mathbb{E}[x_i]=0$, $\|x_i\|\le B$ and $\mathbb{E}[\|x_i\|^2]\le \sigma^2$, the following inequality holds:
\begin{align*}
    \PP\left(\left\|\frac{\sum_{i=1}^nx_i}{n}\right\|\ge \varepsilon\right)\le \exp\left(\frac{1}{4}-\frac{n\varepsilon^2}{8\sigma^2}\right),
\end{align*}
where $\varepsilon\in(0,\frac{\sigma^2}{B})$.
\end{lemma}
\begin{proof}[Proof of Proposition~\ref{prosmo}]
The first property has been proven in \cite[Lemma 2]{bai2022achieving2} and \cite[Lemma 4.4]{yuan2022general}. The bias bound $b^1(H)$ has been established  in \cite[Lemma 4.5]{yuan2022general}.

To establish the upper bound of $b^0(H)$ in the second property, we consider $i \in\{0,\dots,m\}$:
\begin{align*}
        \left|\mathbb{E}\left[\hat V_i^\theta(\rho)\right]- V_i^\theta(\rho)\right|
         \le  \frac{1}{n}\left|\mathbb{E}\left[\sum_{j=1}^{n}\sum_{t=0}^{H-1}\gamma^t r_i(s_t^j,a_t^j)-\sum_{j=1}^{n}\sum_{t=0}^{\infty}\gamma^t r_i(s_t^j,a_t^j)\right]\right|
         \le \sum_{t=H}^{\infty}\gamma^t =\frac{\gamma^{H}}{1-\gamma}.
\end{align*}
Next, we prove that the value function estimator $\hat V_i^\theta(\rho)$ has a sub-Gaussian bound. For $i \in\{0,\dots,m\}$, we note that $\hat{V}^{\theta}_i(\rho)$ is bounded in the interval $[\frac{-1}{1-\gamma},\frac{1}{1-\gamma}]$. Using Hoeffding's inequality, we have for any $\varepsilon>0$,
\begin{align}
    \PP\left(\left|\hat{V}^{\theta}_{i}(\rho)-\mathbb{E}\left[\hat V_i^\theta(\rho)\right]\right|\ge \varepsilon\right)\le 2\exp\left(-\frac{n\varepsilon
    ^2(1-\gamma)^2}{2}\right). \nonumber
\end{align}
We can rewrite the above inequality as:
\begin{align}
    \PP\left(\left|\hat{V}^{\theta}_{i}(\rho)-\mathbb{E}\left[\hat V_i^\theta(\rho)\right]\right|\le \sigma^0(n)\sqrt{\ln{\frac{2}{\delta}}}\right)\ge 1-\delta,\nonumber
    \end{align}
for any $\delta\in(0,1)$, where $\sigma^0(n):=\frac{\sqrt{2}}{\sqrt{n}(1-\gamma)}$. 

Finally, we prove that the gradient estimator $\hat \nabla V_i^\theta(\rho)$ has a sub-Gaussian bound. From \cite[Lemma 4.2]{yuan2022general}, we have
\begin{align*}
     \Var \left[\left(\hat\nabla_\theta{V}^{\theta}_i(\rho)\right)_j\right]\le \frac{M_g^2}{(1-\gamma)^3}.
\end{align*}
We conclude that $\left\|\left(\hat\nabla_\theta{V}^{\theta}_i(\rho)\right)_j\right\|\le \frac{M_g}{(1-\gamma)^2}$ from \cite[Proposition 4.2]{Xu2020Sample} and and $\left\|\mathbb{E} \left(\hat\nabla_\theta{V}^{\theta}_i(\rho)\right)_j\right\|\le \frac{M_g}{(1-\gamma)^2}$ from \cite[Lemma B.1]{liu2020improved}. Therefore, we have
\begin{align*}
   \left\|\left(\hat\nabla_\theta{V}^{\theta}_i(\rho)\right)_j-\mathbb{E}\left[\left(\hat \nabla_\theta V_i^\theta(\rho)\right)_j\right]\right\|\le \frac{2M_g}{(1-\gamma)^2}.
\end{align*}
We can apply Lemma \ref{veber} to the estimator $\hat \nabla_\theta V_i^\theta(\rho):=\frac{1}{n}\sum_{j=1}^{n}\left(\hat \nabla_\theta V_i^\theta(\rho)\right)_j$, yielding the inequality
\begin{align}  
&\PP\left(\left\|\hat\nabla_\theta{V}^{\theta}_i(\rho)-\mathbb{E}\left[\hat \nabla_\theta V_i^\theta(\rho)\right]\right\|\ge \varepsilon\right)
\le\exp\left(\frac{1}{4}-\frac{n\varepsilon^2(1-\gamma)^3}{8M_g^2}\right),\nonumber
\end{align}
where $\varepsilon\in(0,\frac{M_g}{2(1-\gamma)})$. Therefore, for any $\delta\in(0,1)$ and $n\ge8\ln{\frac{\exp{\frac{1}{4}}}{\delta}}$, we have
\begin{align}  
&\PP\left(\left\|\hat\nabla_\theta{V}^{\theta}_i(\rho)-\mathbb{E}\left[\hat \nabla_\theta V_i^\theta(\rho)\right]\right\|\le  \sigma^1(n)\sqrt{\ln{\frac{e^{\frac{1}{4}}}{\delta}}}\right)
\ge 1-\delta,\nonumber
\end{align}
where $\sigma^1(n):=\frac{2\sqrt{2}M_g}{\sqrt{n}(1-\gamma)^{\frac{3}{2}}}$.
\end{proof}
\subsection{Proof of Lemma \ref{gapdelt}}\label{proof_gapdelt}
For the rest of the paper, we first define \(\Delta_t := \hat \nabla_\theta B_\eta^{\theta_t}(\rho) - \nabla_\theta B_\eta^{\theta_t}(\rho)\). Now, we start the proof of Lemma \ref{gapdelt}.
\begin{proof}[Proof of Lemma \ref{gapdelt}]
we can bound $\|\Delta_t\|$ as
\begin{align}
    &\|\Delta_t\|\\
    =&\Biggl\|\nabla_\theta V_0^{\theta_t}(\rho)-\hat\nabla_\theta V_0^{\theta_t}(\rho)+\sum_{i=1}^{m}\biggl[\eta\frac{\hgf-\gf}{\hat \alpha_i(t)}+\eta \gf \left(\frac{1}{\hat{\alpha}_i(t)}-\frac{1}{{\alpha}_i(t)}\right)\biggr]\Biggr\|\nonumber\\
    \le &\left\|\nabla_\theta V_0^{\theta_t}(\rho)-\hat\nabla_\theta V_0^{\theta_t}(\rho)\right\|+\sum_{i=1}^{m}\Biggl[\eta\frac{\left\|\hgf-\gf\right\|}{\hat \alpha_i(t)}+\eta \left\|\gf\right\| \left|\frac{1}{\hat{\alpha}_i(t)}-\frac{1}{{\alpha}_i(t)}\right|\Biggr].\label{bounddel}
\end{align}
Using the sub-Gaussian bound in Proposition \ref{prosmo}, we have 
\begin{align*}
        &\PP\left\{\left|\hat \alpha_i(t)-\alpha_i(t)\right|\le b^0(H)+\sigma^0(n)\sqrt{\ln\frac{2}{\delta}}\right\}\ge 1-\delta,\\
        &\PP\left\{\left\|\hgf- \gf\right\|\le b^1(H)+\sigma^1(n)\sqrt{\ln{\frac{e^{\frac{1}{4}}}{\delta}}}\right\}\ge 1-\delta.
\end{align*}
Also, we know $\|\gf\|\le L$ by Proposition \ref{prosmo}. Combining these properties into inequality \eqref{bounddel}, we have:

\vspace{-0.3cm}
\begin{small}
   \begin{align*}
    \PP\Bigl(\|\hat\nabla_\theta B_\eta^{\theta}(\rho)-\nabla_\theta B_\eta^{\theta}(\rho)\|\le\bigl(1+\sum_{i=1}^{m}\frac{\eta}{\hat V_i^\theta(\rho)}\bigr)\bigl( b^1(H)+\sigma^1(n)\sqrt{\ln{\frac{e^{\frac{1}{4}}}{\delta}}}\bigr)+\sum_{i=1}^{m}\frac{\eta L}{\hat V_i^\theta(\rho) V_i^\theta(\rho)}\allowbreak\bigl( b^0(H)+\sigma^0(n)\sqrt{\ln\frac{2}{\delta}}\bigr)\Bigr)\ge 1-\delta.
\end{align*} 
\end{small}
\vspace{-0.2cm}

Meanwhile, for each $i\in[m]$, we have $\mathbb{P}\left(\frac{3V_i^\theta(\rho)}{4} \leq \hat V_i^\theta(\rho)\right) \geq 1-\delta$ if $\sigma^0(n) \leq \frac{V_i^\theta(\rho)}{8\sqrt{\ln\frac{2}{\delta}}}$ and $b^0(H) \leq \frac{V_i^\theta(\rho)}{8}$, using the sub-Gaussian bound from Proposition \ref{prosmo}. Therefore, to bound the gradient estimation error $\|\hat\nabla_\theta B_\eta^{\theta}(\rho)-\nabla_\theta B_\eta^{\theta}(\rho)\|$ by $\varepsilon$ with probability $1-\delta$, we need to set $\sigma^0(n)$, $\sigma^1(n)$, $b^0(H)$ and $b^1(H)$ as follows:
\begin{align*}
    &\sigma^0(n) \le \min_{i\in[m]}\left\{ \frac{V_i^\theta(\rho)}{8\sqrt{\ln\frac{2}{\delta}}},\frac{3\varepsilon V_i^\theta(\rho)^2}{16m\eta L\sqrt{\ln\frac{2}{\delta}}}\right\},\,\sigma^1(n)\le \min_{i\in[m]}\left\{\frac{3\varepsilon V_i^\theta(\rho)}{4(4m\eta+3V_i^\theta(\rho))\sqrt{\ln{\frac{e^{\frac{1}{4}}}{\delta}}}}\right\}, \\
    &b^0(H) \le \min_{i\in[m]}\left\{ \frac{V_i^\theta(\rho)}{8},\frac{3\varepsilon V_i^\theta(\rho)^2}{16m\eta L}\right\},\, b^1(H)\le \min_{i\in[m]}\left\{\frac{3\varepsilon V_i^\theta(\rho)}{4(4m\eta+3V_i^\theta(\rho))}\right\}.
\end{align*}
To satisfy the above inequality, we can choose $n$ and $H$ as follows using Proposition \ref{prosmo}:
\begin{align*}
    n = \mathcal{O}\left(\frac{\eta^2\ln\frac{1}{\delta}}{(1-\gamma)^6\varepsilon^2\min_{i\in[m]}V_i^\theta(\rho)^4}\right),\, H=\mathcal{O}\left(\ln \frac{1}{\varepsilon \min_{i\in[m]}V_i^\theta(\rho)}\right).
\end{align*}
\end{proof}

\subsection{Determining stepsize and proof of Proposition \ref{small}}
\label{psmall}

\subsubsection{Determining stepsize using local smoothness}
\label{step}
In this section, we discuss the choice of stepsize $\gamma_t$ as introduced in Algorithm \ref{alg:cap}, specifically in line 10. For simplicity, we denote 
\begin{align*}
&\alpha_i(t):=V_i^{\theta_t}(\rho),\,\hat \alpha_i(t):=\hat V_i^{\theta_t}(\rho),\,\underline \alpha_i(t):=\hat V_i^{\theta_t}(\rho)-b^0(H)-\sigma^0(n)\sqrt{\ln{\frac{2}{\delta}}}.
\end{align*}
Recall that the gradient of the log barrier function is defined as $\nabla_\theta B_\eta^\theta(\rho) = \nabla V_{0}^\theta(\rho) + \eta\sum_{i=1}^{m}\frac{\nabla V_{i}^\theta(\rho)}{V_{i}^\theta(\rho)}$. The log barrier function is non-smooth because the norm of the gradient exhibits unbounded growth when $\theta$ approaches the boundary of the feasible domain. However, it has been proven that within a small region around each iterate $\theta_t$, denoted as $\mathcal{R}(\theta_t) = \{\theta\in\Theta \mid V_i^{\theta}(\rho) \geq \frac{V_i^{\theta_t}(\rho)}{2}, i\in[m]\}$, the log barrier function is $M_t$ locally smooth around $\theta_t$ which is defined as
\begin{align*}
    \left\|\nabla B_\eta^{\theta}(\rho)-\nabla B_\eta^{\theta'}(\rho)\right\|\le M_t\|\theta-\theta'\|,\,\forall\, \theta,\theta'\in\mathcal{R}(\theta_t).
\end{align*}
Therefore, we first need to carefully choose the stepsize $\gamma_t$ to ensure that the next iterate $\theta_{t+1}$ remains inside the region $\mathcal{R}(\theta_t)$. In \cite[Lemma 3]{usmanova2022log}, the authors utilize the smoothness of the constraint functions to provide suggestions for choosing $\gamma_t$. We restate the lemma as follows:
\begin{lemma}
Under Assumption \ref{smoli}, if $$\gamma_t\le \min_{i\in[m]}\left\{\frac{\alpha_i(t)}{\sqrt{M\alpha_i(t)}+2|\beta_i(t)|}\right\}\frac{1}{\|\hgb\|},$$ 
we have
$$V_i^{\theta_{t+1}}(\rho)\ge \frac{V_i^{\theta_t}(\rho)}{2}.$$
\end{lemma}
Next, the authors proved the existence of such $M_t$ \cite[Lemma 2]{usmanova2022log}, and we restate it as follows:
\begin{lemma}\label{123}
Let Assumption \ref{smoli} hold, the log barrier function $B_\eta^\theta(\rho)$ is $M_t$ locally smooth for $\theta \in \mathcal{R}(\theta_t)$, where
$$M_t = M +  \sum_{i=1}^{m}\frac{2M\eta }{\alpha_i(t)} + 4\eta\sum_{i=1}^{m}\frac{\left\langle \nabla_\theta V_i^{\theta_{t+1}}(\rho),\frac{\gb}{\left\|\gb\right\|}\right\rangle^2}{\left(\alpha_i(t)\right)^2}.$$
Moreover, if $\gamma_t\le \min_{i\in[m]}\left\{\frac{\alpha_i(t)}{\sqrt{M\alpha_i(t)}+2|\beta_i(t)|}\right\}\frac{1}{\|\hgb\|}$, then
$$M_t = M + \sum_{i=1}^{m}\frac{10 M\eta}{\alpha_i(t)} + 8\eta \sum_{i=1}^{m}\frac{\left(\beta_i(t)\right)^2}{\left(\alpha_i(t)\right)^2},$$
where $\beta_i(t) = \langle\nabla V_i^{\theta_t}(\rho),\frac{\gb}{\left\|\gb\right\|}\rangle$.    
\end{lemma}
Using above lemmas, we set $\gamma_t$ as
\begin{align}
    \gamma_t:=\min\left\{\min_{i\in[m]}\left\{\frac{\alpha_i(t)}{\sqrt{M\alpha_i(t)}+2|\beta_i(t)|}\right\}\frac{1}{\|\hgb\|},\frac{1}{M_t}\right\}.\label{gamma}
\end{align}
Therefore, we can ensure that the next iterate $\theta_{t+1}$ always remains within the region $\mathcal{R}(\theta_t)$, and prevent overshooting by utilizing the local smoothness property. Notice that we only have estimates of $\alpha_i(t)$ and $\beta_i(t)$, therefore we replace $\alpha_i(t)$ with its lower bound of the estimates as $\underline{\alpha}_i(t)$ and $\beta_i(t)$ with its upper bound of the estimates as $\overline{\beta}_i(t):=\left|\langle\hat{\nabla}_\theta \f,\frac{\hgb}{\|\hgb\|}\rangle\right|+\sigma^1(n)\sqrt{\ln{\frac{e^{\frac{1}{4}}}{\delta}}}+  b^1(H)$, $i\in [m]$. Because of the sub-Gaussian bound established in Proposition \ref{prosmo}, $\alpha_i(t)$ is lower bounded by $\underline{\alpha}_i(t)$ and $\beta_i(t)$ is upper by $\overline{\beta}_i(t)$ with high probability. Therefore, we choose the lower bound of \eqref{gamma} to set $\gamma_t$ as
\begin{align*}
    \gamma_t:=\min&\Biggl\{\min_{i\in[m]}\left\{\frac{\underline{\alpha}_i(t)}{\sqrt{M\underline{\alpha}_i(t)}+2|\overline{\beta}_i(t)|}\right\}\frac{1}{\|\hgb\|},\frac{1}{M +\sum_{i=1}^{m}\frac{10M\eta }{\underline{\alpha}_t^i}+8\eta \sum_{i=1}^{m}\frac{\left(\overline{\beta}_t^i\right)^2}{\left(\underline{\alpha}_t^i\right)^2}}\Biggr\}.
\end{align*}
\subsubsection{Proof of Proposition \ref{small}}
To prove Proposition \ref{small}, we establish the following more general lemma, which will be essential for the proof of Theorem \ref{main}.
\begin{lemma}\label{small1}
Define the events $\mathcal{A}$, $\mathcal{B}$, and $\mathcal{C}$ as follows:
\begin{align*}
    \mathcal{A}&:=\left\{\forall t\in[T],\min_{i\in[m]} V_i^{\theta_t}(\rho)\ge c\eta\right\},
    \mathcal{B}:=\left\{\forall t\in[T],\min_{i\in[m]} \gamma_t\ge C\eta\right\},
    \mathcal{C}:=\left\{\forall t\in[T], \left\|\Delta_t\right\|\ge\frac{\eta}{4} \right\},
\end{align*}
\normalsize
\noindent where constants $c$ and $C$ are defined in Equations \eqref{l1} and \eqref{l2}. Let Assumptions~\ref{smoli}, \ref{sl}, and \ref{emf} hold, and set $\eta \leq \nu_{emf}$, $n = \mathcal{O}\left(\frac{\ln\frac{1}{\delta}}{(1-\gamma)^{6+8m} \ell^{4m} \eta^4}\right)$ where {$0<\delta<1$} and {$H = \mathcal{O}\left(\log \frac{1}{(1-\gamma)^{3+4m} \ell^{2m} \eta^2}\right)$}, we have
$$\PP\left\{\mathcal{A}\cap \mathcal{B}\cap \mathcal{C}\right\}\ge 1-mT\delta.$$
\end{lemma}
Our approach to prove Lemma \ref{small1} is as follows: 1) First, we establish $\PP\left(\mathcal{A}\right)\ge 1-\delta$ by considering suitably small variances $\sigma^0(n)$, $\sigma^1(n)$, and biases $b^0(H)$ and $b^1(H)$; 2) Second, we guarantee the event $\mathcal{B}$ based on the construction of $\gamma_t$ in Section \ref{step}; 3) Third, the sub-Gaussian bounds in Proposition \ref{prosmo} enable us to establish $\PP\left(\mathcal{C}\right)\ge 1-\delta$, again by sufficiently bounding the variances $\sigma^0(n)$, $\sigma^1(n)$, and biases $b^0(H)$ and $b^1(H)$. By combining all of these results, we can determine the requirements for the variances $\sigma^0(n)$, $\sigma^1(n)$, and the biases $b^0(H)$ and $b^1(H)$ to satisfy 
\begin{align*}
    &\PP\left\{\mathcal{A}\cap \mathcal{B}\cap \mathcal{C}\right\}\ge 1-\delta.
\end{align*}
All of these factors can be controlled by $n$ and $H$.

To establish 1) above, we first show that the product of the values of $V_i^{\theta_t}(\rho)$, where $i\in \B_\eta(\theta_t)$, does not decrease in the next iteration, as shown in Lemma \ref{ee1}. Lemma \ref{ee1} implies that if one of these constraint values in $\B_\eta(\theta_t)$ decreases in the next step, then at least one of the other constraint values in $\B_\eta(\theta_t)$ will increase in the next step. Therefore, Lemma \ref{ee1} prevents the constraint values from continuously decreasing during the learning process. Furthermore, due to the chosen stepsize, each $V_i^{\theta_{t+1}}(\rho)$ for $i\in[m]$ cannot decrease significantly, as it is always lower bounded by $\frac{V_i^{\theta_{t}}(\rho)}{2}$. Therefore, we can establish 1) as demonstrated in \cite[Lemma 6]{usmanova2022log}.

\begin{lemma}\label{ee1}
 Let Assumptions \ref{smoli}, \ref{sl}, and \ref{emf} hold with $\eta \leq \nu_{emf}$, and set
  \begin{align*}
     &{\sigma}^0(n) \le \frac{\alpha_i(t)\min\left\{2\alpha_i(t),\eta\right\}}{8\eta\sqrt{\ln \frac{2}{\delta}}}, \,{\sigma}^1(n) \le \frac{L\alpha_i(t) }{3\eta\sqrt{\ln \frac{e^{\frac{1}{4}}}{\delta}}},\,b^0(H)\le \frac{\alpha_i(t)\min\left\{2\alpha_i(t),\eta\right\}}{8\eta}, \,b^1(H)\le \frac{L\alpha_i(t)}{3\eta}.
\end{align*}
 \normalsize
If at iteration $t$ we have $\min_{i\in[m]} \alpha_i(t) \leq \frac{\ell\eta}{L(1+\frac{4m}{3})}$, for the next iteration we have $\PP\left(\prod_{i\in\B} \alpha_i({t+1}) \geq \prod_{i\in\B} \alpha_i({t})\right)\allowbreak\ge 1-\delta$ for any $\B$ such that $\B_\eta(\theta_t)\subset \B$, where $\B_\eta(\theta_t):=\left\{i\in[m]\mid \alpha_i(t) \leq \eta \right\}$.
\end{lemma}
\begin{proof}[Proof of Lemma \ref{ee1}]
Due to the choice of stepsize, $\PP(\gamma_t\le \frac{1}{M_t})\ge 1-\delta$, where $M_t$ is the local smoothness constant of the log barrier function $B_\eta^\theta(\rho)$. We have
    \begin{align} 
    &\eta \sum_{i \in \B_\eta(\theta_t)}\log \alpha_i({t+1})-\eta \sum_{i \in \B_\eta(\theta_t)}\log \alpha_i(t)\nonumber\\
     \ge& \gamma_t\left\langle\eta \sum_{i \in \B_\eta(\theta_t)} \frac{\gf}{\alpha_i(t)},\hgb \right\rangle-\frac{M_t \gamma_t^2}{2} \left\|\hgb\right\|^2 \nonumber\\
     \ge& \gamma_t\left(\left\langle\eta \sum_{i \in \B_\eta(\theta_t)} \frac{\gf}{\alpha_i(t)}, \hgb\right\rangle-\frac{1}{2}\left\|\hgb\right\|^2\right)\nonumber \\ 
     =&\frac{\gamma_t \eta^2}{2}\left(2\langle D_1, D_1+D_2\rangle-\|D_1+D_2\|^2\right) \nonumber\\ 
     =&\frac{\gamma_t \eta^2}{2}\left(\|D_1\|^2-\|D_2\|^2\right),\label{boundalpha}
    \end{align}
 where $D_1:= \sum_{i \in \B_\eta(\theta_t)} \frac{\gf}{\alpha_i(t)}$ and $D_2:=\frac{\hgb}{\eta}-\sum_{i \in \B_\eta(\theta_t)} \frac{\gf}{\alpha_i(t)}$. Under Assumption ~\ref{emf}, we have
    \begin{align}
        \|D_1\|&=\left\| \sum_{i \in \B_\eta(\theta_t)} \frac{\gf}{\alpha_i(t)}\right\|\ge \langle \sum_{i \in \B_\eta(\theta_t)} \frac{\gf}{\alpha_i(t)},s_\theta\rangle\ge  \sum_{i \in \B_\eta(\theta_t)}\frac{\ell}{\alpha_i(t)}\ge \frac{L(1+\frac{4m}{3})}{\eta},\label{A}
    \end{align}
    where we use $\min_{i\in[m]} \alpha_i(t) \leq \frac{\ell\eta}{L(1+\frac{4m}{3})}$ in the last step. For each $i\in[m]$, since $\sigma^0(n) \leq \frac{\alpha_i(t)}{8\sqrt{\ln\frac{2}{\delta}}}$ and $b^0(H) \leq \frac{\alpha_i(t)}{8}$, we have $\mathbb{P}\left(\frac{3\alpha_i(t)}{4} \leq \hat \alpha_i(t)\right) \geq 1-\delta$ using the sub-Gaussian bound in Proposition \ref{prosmo}. Therefore, $\mathbb{P}\left(\frac{3\eta}{4} \leq \hat \alpha_i(t)\right) \geq 1-\delta$ for $i \notin\B_\eta(\theta_t)$. Then, we can upper bound $\|D_2\|$ with probability at least $1-\delta$ as follows:
 \begin{align}
    &\|D_2\|\\
    =&\Biggl\|\frac{\hat \nabla_\theta V_0^{\theta_t}(\rho)}{\eta}+\sum_{i \notin\B_\eta(\theta_t)} \frac{\hgf}{\hat{\alpha}_i(t)}+\sum_{i \in \B_\eta(\theta_t)} \frac{\hgf}{\hat{\alpha}_i(t)}-\sum_{i \in \B_\eta(\theta_t)} \frac{\gf}{\alpha_i(t)}\Biggr\| \label{i1}\\
     \leq & \frac{L}{\eta}\left(1+\frac{4}{3}\left(m-\left|\B_\eta(\theta_t)\right|\right)\right)+\Biggl\|\sum_{i \in \B_\eta(\theta_t)}\biggl( \frac{\hgf}{\hat{\alpha}_i(t)}-\frac{\hgf}{{\alpha}_i(t)}+\frac{\hgf}{{\alpha}_i(t)}-\frac{\gf}{\alpha_i(t)}\biggr)\Biggr\|\label{i2}\\
     \leq & \sum_{i \in \B_\eta(\theta_t)}\Biggl(\frac{\left\|\hgf-\gf\right\|}{\alpha_i(t)}+\left|\frac{1}{\hat{\alpha}_i(t)}-\frac{1}{\alpha_i(t)}\right| \left\|\hgf \right\|\Biggr)+\frac{L}{\eta}\left(1+\frac{4}{3}\left(m-\left|\B_\eta(\theta_t)\right|\right)\right)\label{i3}\\
     \leq & \sum_{i \in \B_\eta(\theta_t)}\left(\frac{{\sigma}^1(n) \sqrt{\ln \frac{e^{\frac{1}{4}}}{\delta}}+ b^1(H)}{\alpha_i(t)}+L\frac{{\sigma}^0(n) \sqrt{\ln \frac{2}{\delta}}+ b^0(H)}{\hat{\alpha}_i(t)\alpha_i(t)}\right) + \frac{L}{\eta}\left(1+\frac{4}{3}\left(m-\left|\B_\eta(\theta_t)\right|\right)\right)\label{i4}\\
    \leq & \sum_{i \in \B_\eta(\theta_t)}\left(\frac{{\sigma}^1(n) \sqrt{\ln \frac{e^{\frac{1}{4}}}{\delta}}+ b^1(H)}{\alpha_i(t)}+4L\frac{{\sigma}^0(n) \sqrt{\ln \frac{2}{\delta}}+ b^0(H)}{3(\alpha_i(t))^2}\right)+\frac{L}{\eta}\left(1+\frac{4}{3}\left(m-\left|\B_\eta(\theta_t)\right|\right)\right).\label{i5}
    \end{align}     
    \normalsize
From \eqref{i1} to \eqref{i2}, we use $\frac{3\eta}{4} \leq \hat \alpha_i(t)$ and $\|\hat \nabla_\theta V_0^{\theta_t}(\rho)\|\le L$ by \cite[Proposition 4.2]{Xu2020Sample}. From \eqref{i3} to  \eqref{i4}, we use the sub-Gaussian bound in Proposition \ref{prosmo} and $\|\hat \nabla_\theta V_i^{\theta_t}(\rho)\|\le L$ for $i\in[m]$. From  \eqref{i4} to  \eqref{i5}, we use $\frac{3\alpha_i(t)}{4} \leq \hat \alpha_i(t)$. Further, if we set the variances and biases in \eqref{i5} as
\begin{align*}
     &{\sigma}^0(n) \le \frac{(\alpha_i(t))^2}{4\eta\sqrt{\ln \frac{2}{\delta}}},\,{\sigma}^1(n) \le \frac{L\alpha_i(t) }{3\eta\sqrt{\ln \frac{e^{\frac{1}{4}}}{\delta}}},\, b^0(H)\le \frac{(\alpha_i(t))^2}{4\eta},\,b^1(H)\le \frac{L\alpha_i(t)}{3\eta},
\end{align*}
then we can have $\PP\left(\|D_2\|\le \frac{L(1+\frac{4m}{3})}{\eta}\right)\ge 1-\delta$. Combining this property with ~\eqref{A}, we have $\PP\left(\|D_2\|\le \|D_1\|\right)\ge 1-\delta$. Taking this relation into  ~\eqref{boundalpha}, we have 
\begin{align*}
    \PP\left(\prod_{i\in\B_\eta(\theta_t)} \alpha_i({t+1}) \geq \prod_{i\in\B_\eta(\theta_t)} \alpha_i({t})\right)\ge 1-\delta.
\end{align*}
Same result if we replace $\B_\eta(\theta_t)$ with any $\B$ such that $\B_\eta(\theta_t) \subset \B$.
\end{proof}
With the above lemma in place, we are ready to prove Lemma \ref{small1}. 
\begin{proof}[Proof of Lemma \ref{small1}]
First, we need to choose $\eta \leq \nu_{emf}$ and set
\begin{align*}
     &{\sigma}^0(n) \le \frac{\alpha_i(t)\min\left\{2\alpha_i(t),\eta\right\}}{8\eta\sqrt{\ln \frac{2}{\delta}}}, \,{\sigma}^1(n) \le \frac{L\alpha_i(t) }{3\eta\sqrt{\ln \frac{e^{\frac{1}{4}}}{\delta}}},\,b^0(H)\le \frac{\alpha_i(t)\min\left\{2\alpha_i(t),\eta\right\}}{8\eta}, \,b^1(H)\le \frac{L\alpha_i(t)}{3\eta}
\end{align*}
to satisfy the conditions in Lemma \ref{ee1}. Then, we can combine the result from Lemma \ref{ee1} with the result from \cite[Lemma 6]{usmanova2022log}, and we have 
\begin{align}
    &\mathbb{P}\biggl\{\forall t \in [T], \min_{i\in [m]}\alpha_i(t)\geq c\eta,\,\min_{i\in [m]}\hat \alpha_i(t)\geq \frac{3}{8}c\eta\text{ and }\min_{i\in [m]}\underline\alpha_i(t)\geq \frac{c\eta}{2} \biggr\} \geq 1-mT\delta,\, c=\left(\frac{\ell(1-\gamma)^2}{4M_g(1+\frac{4m}{3})}\right)^m.\label{l1}
\end{align}
\normalsize
Based on the lower bound of $\hat \alpha_i(t)$, we can further bound $\gamma_t$ by \cite[Lemma 6]{usmanova2022log} as
\begin{align}
    \gamma_t \ge C\eta ,\,C:=\frac{c}{2L^2(1+\frac{m}{c})\max\left\{4+\frac{5Mc}{L^2},1+\sqrt{\frac{Mc}{4L^2}}\right\}}.\label{l2}
\end{align}
Further, if we set 
\begin{align}
     &{\sigma}^0(n) \le \min\left\{ \frac{\alpha_i(t)\min\left\{2\alpha_i(t),\eta\right\}}{8\eta\sqrt{\ln \frac{2}{\delta}}},\frac{1}{\left(\sum_{i=1}^{m}\frac{ 16L}{\alpha_i(t)\hat \alpha_i(t)}\right)\sqrt{\ln \frac{2}{\delta}}}\right\},\\
     &{\sigma}^1(n) \le\min\left\{ \frac{L\alpha_i(t) }{3\eta\sqrt{\ln \frac{e^{\frac{1}{4}}}{\delta}}},\frac{\eta}{16\left(1+\sum_{i=1}^{m}\frac{\eta}{\hat\alpha_i(t)}\right)\sqrt{\ln \frac{e^{\frac{1}{4}}}{\delta}}}\right\},\nonumber\\
     &b^0(H)\le \min\left\{ \frac{\alpha_i(t)\min\left\{2\alpha_i(t),\eta\right\}}{8\eta},\frac{1}{16\left(\sum_{i=1}^{m}\frac{ L}{\alpha_i(t)\hat \alpha_i(t)}\right)}\right\},b^1(H)\le\min\left\{ \frac{L\alpha_i(t) }{3\eta},\frac{\eta}{16\left(1+\sum_{i=1}^{m}\frac{\eta}{\hat\alpha_i(t)}\right)}\right\}.\label{vb}
\end{align}
By Lemma \ref{gapdelt}, we have
\begin{align}
    \PP\left(\|\Delta_t\|\ge \frac{\eta}{4}\right),  \label{bounddel2}
\end{align}
Combing the results from \eqref{l1} and \eqref{l2} with inequality \eqref{bounddel2}, we have 
\begin{align*}
    &\PP\left\{\forall t\in[T],\,\min_{i\in[m]} V_i^{\theta_t}(\rho)\ge c\eta ,\,\gamma_t \ge C\eta\text{ and }\left\|\Delta_t\right\|\ge\frac{\eta}{4}\right\}
    \ge 1-mT\delta.
\end{align*}
Based on the above result regarding the lower bound on $\alpha_i(t)$ and $\hat\alpha_i(t)$, we can further set the variances and biases in \eqref{vb} as follows:
\begin{align}
     &{\sigma}^0(n) \le \min\left\{ \frac{c\eta\min\left\{4c,1\right\}}{4\sqrt{\ln \frac{2}{\delta}}},\frac{3c^2\eta^2}{32L\sqrt{\ln \frac{2}{\delta}}}\right\},{\sigma}^1(n) \le\min\left\{ \frac{2cL }{3\sqrt{\ln \frac{e^{\frac{1}{4}}}{\delta}}},\,\frac{\eta}{16\left(1+\frac{4m}{3c}\right)\sqrt{\ln \frac{e^{\frac{1}{4}}}{\delta}}}\right\},\nonumber\\
     &b^0(H) \le\min\left\{ \frac{c\eta\min\left\{4c,1\right\}}{4},\frac{3c^2\eta^2}{32L}\right\},b^1(H)\le\min\left\{ \frac{2cL }{3},\frac{\eta}{16\left(1+\frac{4m}{3c}\right)}\right\}.\nonumber
\end{align}
According to Proposition \ref{prosmo}, the number of trajectories $n$ and the truncated horizon $H$ need to be set as follows:
\begin{align}
n:=\max&\biggl\{\frac{2048L^2{\ln \frac{2}{\delta}}}{9(1-\gamma)^2c^4\eta^4},\frac{32{\ln \frac{2}{\delta}}}{c^2(1-\gamma)^2\eta^2\min\left\{16c^2,1\right\}},\frac{2048(1+\frac{4m}{3c})^2 M_g^2\ln\frac{e^{\frac{1}{4}}}{\delta}}{\eta^2(1-\gamma)^3}\frac{18 M_g^2\ln\frac{e^{\frac{1}{4}}}{\delta}}{c^2L^2(1-\gamma)^3}\biggr\},\nonumber\\
H:=\max&\biggl\{\log_\gamma\left(\frac{3(1-\gamma)c^2\eta^2}{32L}\right),\log_\gamma\left(\frac{c(1-\gamma)\eta\min\left\{4c,1\right\}}{4}\right),\mathcal{O}\left(\log_\gamma \frac{(1-\gamma)\eta}{16(1+\frac{4m}{3c})M_g}\right),\\
&\mathcal{O}\left(\log_\gamma \frac{2(1-\gamma)cL}{3M_g}\right)\biggr\},\tag{N-H}\label{nh}
\end{align}
where the first condition for $H$ is to satisfy $b^0(H) \le\min\left\{ \frac{c\min\left\{4c,1\right\}}{4},\frac{3c^2}{32L}\right\}$ and the second condition for $H$ is to satisfy $b^1(H)\le\min\left\{ \frac{2cL }{3},\frac{\eta}{16\left(1+\frac{4m}{3c}\right)}\right\}$. Therefore, we need to set the number of trajectories $n$ of the order $\mathcal{O}\left(\frac{\ln\frac{1}{\delta}}{(1-\gamma)^{6+8m} \ell^{4m} \eta^4}\right)$ and the truncated horizon $H$ of the order $\mathcal{O}\left(\ln \frac{1}{ \ell\eta}\right)$. 
\end{proof}

\subsection{Proof of Lemma~\ref{gdm}}
\label{proofgdm}
The proof of Lemma \ref{gdm} is based on the performance difference lemma, provided below for completeness.
\begin{theorem}[The performance difference lemma \cite{sutton1999policy}] \label{3}$\forall \theta,\theta'\in \R^d$, $\forall i \in\{0,\dots,m\}$, we have
\begin{align}
V^\theta_i(\rho)-V^{\theta'}_i(\rho)=\frac{1}{1-\gamma}\mathbb{E}_{(s,a)\sim d_\rho^\theta}\left[A_i^{{\theta'}}(s,a)\right].\nonumber
\end{align}
\end{theorem}
\begin{proof}[Proof of Lemma \ref{gdm}]
We can derive the following equality by using the performance difference lemma,
\begin{align}
  &V_{0}^{\pi^*}\left(\rho\right)-V_{0}^{\theta}\left(\rho\right)+\eta\sum_{i=1}^{m}\left(\frac{V_{i}^{\pi^*}(\rho)-V_{i}^{\theta}\left(\rho\right)}{ V_{i}^{\theta}\left(\rho\right)}\right)= \frac{1}{1-\gamma}\mathbb{E}_{(s,a)\sim d_\rho^{\pi^*}}\left[A_0^{{\theta}}(s,a)+\eta \sum_{i=1}^{m}\frac{A_i^{{\theta}}(s,a)}{ V_{i}^{\theta}\left(\rho\right)}\right].\label{s1}
\end{align}
Applying Jensen's inequality to Assumption \ref{bae}, we obtain $\forall i \in\{0,\dots,m\}$,
\begin{align}
  &\mathbb{E}_{(s,a)\sim d_\rho^{\pi^*}}\left[{A}_i^\theta(s,a)-(1-\gamma){\mu_i^*}^T \nabla_\theta\log\pi_\theta(a|s)\right]
  \le\sqrt{\varepsilon_{bias}}.\label{s2}
\end{align}
Plugging inequality \eqref{s2} into \eqref{s1}, we get
\begin{align}
    &V_{0}^{\pi^*}\left(\rho\right)-V_{0}^{\theta}\left(\rho\right)+\eta\sum_{i=1}^{m}\left(\frac{V_{i}^{\pi^*}(\rho)-V_{i}^{\theta}\left(\rho\right)}{ V_{i}^{\theta}\left(\rho\right)}\right)\nonumber\\
    \le&\mathbb{E}_{(s,a)\sim d_\rho^{\pi^*}}\left[\left({\mu_0^*}+\sum_{i=1}^{m}\frac{\eta \mu_i^*}{ V_{i}^{\theta}\left(\rho\right)}\right)^T \nabla_\theta\log\pi_{\theta}(a|s)\right]+\frac{\sqrt{\varepsilon_{bias}}}{1-\gamma}\left(\sum_{i=1}^{m}\frac{\eta}{ V_{i}^{\theta}\left(\rho\right)}+1\right)\nonumber\\
    =& \mathbb{E}_{(s,a)\sim d_\rho^{\pi^{*}}}\left[\left(\nabla_\theta B_\eta^{\theta}(\rho)\right)^T \left(F^{\theta}(\rho)\right)^{\dagger} \nabla_\theta\log\pi_{\theta}(a|s)\right]+\frac{\sqrt{\varepsilon_{bias}}}{1-\gamma}\left(1+\sum_{i=1}^{m}\frac{\eta}{V_{i}^{\theta}\left(\rho\right)}\right)\nonumber\\
    =&\mathbb{E}_{(s,a)\sim d_\rho^{\pi^{*}}}\left[\left(\mathbf{P}_{\mathbf{Ker}(F^\theta(\rho))}\nabla_\theta B_\eta^{\theta}(\rho)+\mathbf{P}_{\mathbf{Im}(F^\theta(\rho))}\nabla_\theta B_\eta^{\theta}(\rho)\right)^T \left(F^{\theta}(\rho)\right)^{\dagger} \nabla_\theta\log\pi_{\theta}(a|s)\right]+\frac{\sqrt{\varepsilon_{bias}}}{1-\gamma}\left(1+\sum_{i=1}^{m}\frac{\eta}{V_{i}^{\theta}\left(\rho\right)}\right)\nonumber\\
    =&\mathbb{E}_{(s,a)\sim d_\rho^{\pi^{*}}}\left[\left(\mathbf{P}_{\mathbf{Im}(F^\theta(\rho))}\nabla_\theta B_\eta^{\theta}(\rho)\right)^T \left(F^{\theta}(\rho)\right)^{\dagger} \nabla_\theta\log\pi_{\theta}(a|s)\right]+\frac{\sqrt{\varepsilon_{bias}}}{1-\gamma}\left(1+\sum_{i=1}^{m}\frac{\eta}{V_{i}^{\theta}\left(\rho\right)}\right)\nonumber\\
    \le &\frac{M_h}{\mu_{F}}\left\| \mathbf{P}_{\mathbf{Im}(F^\theta(\rho))}\nabla_\theta B_\eta^{\theta}(\rho)\right\|+\frac{\sqrt{\varepsilon_{bias}}}{1-\gamma}\left(1+\sum_{i=1}^{m}\frac{\eta}{V_{i}^{\theta}\left(\rho\right)}\right),
\end{align}
where we use Assumption \ref{fn} to bound $\left\| \left(\mathbf{P}_{\mathbf{Im}(F^\theta(\rho))}\nabla_\theta B_\eta^{\theta}(\rho)\right)^T \left(F^{\theta}(\rho)\right)^{\dagger}\right\|$ by ${\left\| \mathbf{P}_{\mathbf{Im}(F^\theta(\rho))}\nabla_\theta B_\eta^{\theta}(\rho)\right\|}/{\mu_F}$ and Assumption \ref{smoli} to bound $\left\|\nabla\log\pi_{\theta}(a|s)\right\|$ by $M_h$ in the last step.
Rearranging the above inequality, we have
\begin{align}
    V_{0}^{\pi^*}\left(\rho\right)-V_{0}^{\theta}\left(\rho\right)
    \le& m\eta+ \frac{\sqrt{\varepsilon_{bias}}}{1-\gamma}\left(1+\sum_{i=1}^{m}\frac{\eta}{ V_{i}^{\theta}\left(\rho\right)}\right)+\frac{M_h}{\mu_{F}}\left\| \mathbf{P}_{\mathbf{Im}(F^\theta(\rho))}\nabla_\theta B_\eta^{\theta}(\rho)\right\|\nonumber\\
    \le& m\eta+ \frac{\sqrt{\varepsilon_{bias}}}{1-\gamma}\left(1+\sum_{i=1}^{m}\frac{\eta}{ V_{i}^{\theta}\left(\rho\right)}\right)+\frac{M_h}{\mu_F}\left\| \nabla_\theta B_\eta^{\theta}(\rho)\right\|.\nonumber
\end{align}
\end{proof}
\subsection{Proof of Lemma \ref{stationl}}
\label{station}
In this section, we prove that the stationary points of the log barrier function is at most $\Omega(\nu_{emf}+\eta)$ close to the boundary.
\begin{proof}[Proof of Lemma \ref{stationl}]
Since $\theta_{\text{st}}$ is the stationary point, we have
\begin{align}
    \nabla_\theta B_\eta^{\theta_{\text{st}}}(\rho)&=\nabla_\theta V_0^{\theta_{\text{st}}}(\rho)+\eta\sum_{i=1}^{m}\frac{\nabla_\theta V_i^{\theta_{\text{st}}}(\rho)}{ V_i^{\theta_{\text{st}}}(\rho)}=0. \nonumber
\end{align}
Rearranging the terms in the above equation, we obtain
\begin{align}
     &\sum_{i\notin \B_{\nu_{emf}}(\theta_{\text{st}})}\frac{\nabla_\theta V_i^{\theta_{\text{st}}}(\rho)}{ V_i^{\theta_{\text{st}}}(\rho)}+\sum_{i\in \B_{\nu_{emf}}(\theta_{\text{st}})}\frac{\nabla_\theta V_i^{\theta_{\text{st}}}(\rho)}{ V_i^{\theta_{\text{st}}}(\rho)}=\frac{-\nabla_\theta V_0^{\theta_{\text{st}}}(\rho)}{\eta}.\label{gap2}
\end{align}
If $\B_{\nu_{emf}}(\theta_{\text{st}})$ is an empty set, then
$$\min_{i\in[m]}\left\{V_i^{\theta_{\text{st}}}(\rho)\right\}\ge {\nu_{emf}}.$$
Otherwise, by Assumption \ref{emf}, there exists a unit vector $s_{\theta_{\text{st}}}\in\R^d$ such that for $i\in \B_{\nu_{emf}}(\theta_{\text{st}})$, we have
$$\langle s_{\theta_{\text{st}}}, \nabla_\theta V_i^{\theta_{\text{st}}}(\rho)\rangle >\ell.$$
Taking the dot product of both sides of equation \eqref{gap2} with $s_{\theta_{\text{st}}}$ and using Lipschitz continuity, we obtain
\begin{align*}
    \frac{\ell}{\min_{i\in[m]}\left\{V_i^{\theta_{\text{st}}}(\rho)\right\}}
    \le &\frac{\left\langle s_{\theta_{\text{st}}}, \nabla_\theta V_i^{\theta_{\text{st}}}(\rho)\right\rangle}{\min_{i\in[m]}\left\{ V_i^{\theta_{\text{st}}}(\rho)\right\}}\sum_{i\in \B_{\nu_{emf}}(\theta_{\text{st}})}\frac{\min_{i\in[m]}\left\{V_i^{\theta_{\text{st}}}(\rho)\right\}}{V_i^{\theta_{\text{st}}}(\rho)}\\
    =&\frac{\left\langle s_{\theta_{\text{st}}},-\nabla_\theta V_0^{\theta_{\text{st}}}(\rho)\right\rangle}{\eta}-\sum_{i\notin \B_{\nu_{emf}}(\theta_{\text{st}})}\frac{\left\langle s_{\theta_{\text{st}}}, \nabla_\theta V_i^{\theta_{\text{st}}}(\rho)\right\rangle}{V_i^{\theta_{\text{st}}}(\rho)} 
    \ge \frac{mL}{\min\{\eta,\nu_{emf}\}},\nonumber
\end{align*}
Therefore, 
$$\min_{i\in[m]}\left\{ V_i^{\theta_{\text{st}}}(\rho)\right\}\ge \frac{\min\{\eta,\nu_{emf}\} \ell}{mL}.$$
In conclusion, we have $\min_{i\in[m]}\left\{ V_i^{\theta_{\text{st}}}(\rho)\right\}\ge \min\left\{\frac{{\nu_{emf}},\min\{\eta,\nu_{emf}\} \ell}{mL}\right\}$.
\end{proof}
\subsection{Proof of Theorem \ref{main}}\label{proofmain}
\begin{proof}[Proof of Theorem \ref{main}:] We set the values for $n$, $H$, and $\eta$ to satisfy the conditions outlined in Proposition \ref{small}. Due to our choice of stepsize, we have {$\PP\left(\gamma_t \leq \frac{1}{M_{t}}\right) \geq 1-\delta$} as showed in Proposition \ref{small}, where {$M_{t}$} represents the local smoothness constant of the log barrier function {$B_\eta^\theta(\rho)$}. With this, we can bound {$B_\eta^{\theta_{t+1}}(\rho) - B_\eta^{\theta_t}(\rho)$} with probability at least $1-\delta$ as
    \begin{align}
        B_\eta^{\theta_{t+1}}(\rho) - B_\eta^{\theta_t}(\rho)
        \geq& \gamma_t \left\langle \nabla_\theta B_\eta^{\theta_t}(\rho),\hat \nabla_\theta B_\eta^{\theta_t}(\rho)\right\rangle - \frac{M_{t} \gamma_t^2}{2}\left\|\hat \nabla_\theta B_\eta^{\theta_t}(\rho)\right\|^2\nonumber\\
        \geq& \gamma_t \left\langle \nabla_\theta B_\eta^{\theta_t}(\rho),\hat \nabla_\theta B_\eta^{\theta_t}(\rho)\right\rangle - \frac{\gamma_t}{2}\left\|\hat \nabla_\theta B_\eta^{\theta_t}(\rho)\right\|^2\nonumber\\
        =& \gamma_t \left\langle \nabla_\theta B_\eta^{\theta_t}(\rho),\Delta_t + \nabla_\theta B_\eta^{\theta_t}(\rho)\right\rangle - \frac{\gamma_t}{2}\left\|\Delta_t + \nabla_\theta B_\eta^{\theta_t}(\rho)\right\|^2\nonumber\\
        =&\frac{\gamma_t}{2}\left\| \nabla_\theta B_\eta^{\theta_t}(\rho)\right\|^2-\frac{\gamma_t}{2}\left\|\Delta_t\right\|^2.\label{ss_1}
    \end{align}
    We divide the analysis into two cases based on the \textbf{if condition} in algorithm \ref{alg:cap} line 5.
    
    \textbf{Case 1:} If {$\|\hat \nabla_\theta B_\eta^{\theta_t}(\rho)\|\ge \frac{\eta}{2}$}, then {$\|\nabla_\theta B_\eta^{\theta_t}(\rho)\|\ge \frac{\eta}{4}$} since $\|\Delta_t\|\le \frac{\eta}{4}$ by Proposition \ref{small}. We can further write \eqref{ss_1} as
    \begin{align}
    B_\eta^{\theta_{t+1}}(\rho) - B_\eta^{\theta_t}(\rho) &\ge \frac{C\eta^2}{8}\left\|\nabla_\theta B_\eta^{\theta_t}(\rho)\right\| - \frac{C\eta^3}{32},\label{case1:1}
    \end{align}
    where we plug in $\gamma_t\ge C\eta$ in the last step. Since $\min_{i\in[m]} V_i^{\theta_t}(\rho)\geq c\eta$, we can rewrite Lemma \ref{gdm} as:
    \begin{align}
        V_{0}^{\pi^*}(\rho) - V_{0}^{\theta_t}(\rho)
        \le a + \frac{M_h}{\mu_F}\left\| \nabla_\theta B_\eta^{\theta_t}(\rho)\right\|,\label{gmdp}
     \end{align}
    where $a = m\eta + \frac{\sqrt{\varepsilon_{bias}}}{1-\gamma}\left(1+\frac{m}{c}\right)$. Plugging \eqref{gmdp} into \eqref{case1:1}, we get
    \begin{align}
    B_\eta^{\theta_{t+1}}(\rho) - B_\eta^{\theta_t}(\rho) 
    \ge \frac{C\mu_F\eta^2}{8M_h}\left(V_{0}^{\pi^*}(\rho) - V_{0}^{\theta_t}(\rho)\right) - \frac{C\eta^3}{32} - \frac{aC\mu_F\eta^2 }{8M_h},\nonumber
    \end{align}
    which can be further simplified to:
    \begin{align}
    &V_{0}^{\pi^*}(\rho) - V_{0}^{\theta_{t+1}}(\rho)
    \le \left(1-\frac{C\mu_F\eta^2}{8M_h}\right)\left(V_{0}^{\pi^*} (\rho)- V_{0}^{\theta_t}(\rho)\right) + \frac{C\eta^3}{32}  + \frac{aC\mu_F\eta^2 }{8M_h} + \eta \sum_{i=1}^{m}\log\frac{V_i^{\theta_{t+1}}(\rho)}{V_i^{\theta_{t}}(\rho)}.\nonumber
    \end{align}
    \normalsize
    By recursively applying the above inequality and setting $\frac{C\mu_F\eta^2}{8M_h}<1$, we obtain 
        \begin{align}
        &V_{0}^{\pi^*}(\rho)-V_{0}^{\theta_{t+1}}(\rho)\nonumber\\
        \le&\left(1- \frac{C\mu_F\eta^2}{8M_h}\right)^{t+1}\left(V_{0}^{\pi^*}(\rho)-V_{0}^{\theta_{0}}(\rho)\right)
        +\left(\frac{aC\mu_F\eta^2}{8M_h}+\frac{C\eta^3}{32}\right)\sum_{i=0}^{t}\left(1- \frac{C\mu_F\eta^2}{8M_h}\right)^i\nonumber\\
        & +\eta\sum_{i=1}^{m}\log V_i^{\theta_{t+1}}(\rho)-\eta\left(1- \frac{C\mu_F\eta^2}{8M_h}\right)^{t}\sum_{i=1}^{m}\log V_i^{\theta_{0}}(\rho) -\frac{C\mu_F\eta^3}{8M_h}\sum_{i=1}^{m}\sum_{j=1}^{t}\left(1- \frac{C\mu_F\eta^2}{8M_h}\right)^{t-j} \log{V_i^{\theta_j}(\rho)} \nonumber\\
        \le& \left(1- \frac{C\mu_F\eta^2}{8M_h}\right)^{t+1}\left(V_{0}^{\pi^*}(\rho)-V_{0}^{\theta_{0}}(\rho)\right)+a+\frac{M_h\eta}{4\mu_F}+m\eta\log \frac{1}{1-\gamma} \nonumber\\
        &-m\eta\left(1- \frac{C\mu_F\eta^2}{8M_h}\right)^{t}\log \nu_s-\frac{C\mu_F\eta^3}{8M_h}\sum_{i=1}^{m}\sum_{j=1}^{t}\left(1- \frac{C\mu_F\eta^2}{8M_h}\right)^{t-j} \log(c\eta)\nonumber\\
        \le& \left(1- \frac{C\mu_F\eta^2}{8M_h}\right)^{t+1}\left(V_{0}^{\pi^*}(\rho)-V_{0}^{\theta_{0}}(\rho)\right)+a+\frac{M_h\eta}{4\mu_F}+m\eta\log\frac{1}{c\eta(1-\gamma)}-m\eta\left(1- \frac{C\mu_F\eta^2}{8M_h}\right)^{t}\log \nu_s.\label{m1}
    \end{align}
    \textbf{Case 2:} If {$ \|\hat \nabla_\theta B_\eta^{\theta_t}(\rho)\|\le \frac{\eta}{2}$}, we have {$\| \nabla_\theta B_\eta^{\theta_t}(\rho)\|\le  \|\hat \nabla_\theta B_\eta^{\theta_t}(\rho)\|+\|\Delta_t\|\le \frac{3\eta}{4}$}. Applying \eqref{gmdp}, we have
    \begin{align}
        &V_{0}^{\pi^*} (\rho)- V_{0}^{\theta_t}(\rho)
        \le \frac{3M_h\eta}{4\mu_F}+a.
    \label{m2}
    \end{align}
    Combining the inequalities \eqref{m1} and \eqref{m2}, we conclude that after $T$ iterations of the Algorithm \ref{alg:cap}, the output policy $\pi_{\theta_{\text{out}}}$ satisfies
            \begin{align*}
                &V_0^{\pi^*}(\rho) - V_0^{ \theta_{\text{out}}}(\rho)\\
                \le& \left(1- \frac{C\mu_F\eta^2}{8M_h}\right)^{T}\left(V_{0}^{\pi^*} (\rho)- V_{0}^{\theta_{0}}(\rho)\right) + \frac{\sqrt{\varepsilon_{bias}}}{1-\gamma}\left(1+\frac{m}{c}\right)\\
                &+m\eta\left(\frac{3M_h}{4\mu_F m}+1+\log\frac{1}{c\eta(1-\gamma)}-\left(1- \frac{C\mu_F\eta^2}{8M_h}\right)^{T-1}\log \nu_s\right)\\
                =& \mathcal{O}\left(\exp{\left(-CT\mu_F\eta^2\right)}\right)\left(V_{0}^{\pi^*} (\rho)- V_{0}^{\theta_{0}}(\rho)\right) +\mathcal{O}\left(\frac{\sqrt{\varepsilon_{\text{bias}}}}{(1-\gamma)^{2m+1}\ell^m}\right)  + {\mathcal{O}}\left(\frac{\eta}{\mu_F}\right)+\mathcal{O}(\eta\ln\frac{1}{\eta})\nonumber
            \end{align*}
     \normalsize
    with a probability of at least $1-mT\delta$.
\end{proof}

\subsection{Proof of Corollary \ref{cor1}}\label{proofcor1}
\begin{proof}[Proof of Corollary \ref{cor1}]
    By analyzing the inequalities, namely \eqref{m1} and \eqref{m2}, provided in Section \ref{proofmain}, we can prove Corollary \ref{cor1} as follows: setting $n = \mathcal{O}\left(\frac{\ln\frac{1}{\delta}}{(1-\gamma)^{6+8m} \ell^{4m} \varepsilon^4}\right)$ and {$H = \mathcal{O}\left(\log \frac{1}{\ell \varepsilon}\right)$}, $\eta = \varepsilon$ and $T = \mathcal{O}\left(\frac{\log\frac{1}{\varepsilon}}{\mu_F\varepsilon^2 \ell^{2m} (1-\gamma)^{4m+4}}\right)$. After $T$ iterations of the LB-SGD Algorithm, the output $\theta_{\text{out}}$ satisfies
\begin{align}
    V_0^{\pi^*}(\rho) - V_0^{\theta_{\text{out}}}(\rho) \le {\mathcal{O}}\left(\varepsilon\log\frac{1}{\varepsilon}\right) +\mathcal{O}\left( \frac{\varepsilon}{\mu_F}\right) + \mathcal{O}\left(\frac{\sqrt{\varepsilon_{\text{bias}}}}{(1-\gamma)^{2m+1}\ell^m}\right),\label{fila}
\end{align}
while safe exploration is ensured with a probability of at least $1-mT\delta$. To ensure satisfaction of \eqref{fila} and maintain safe exploration with a probability of at least $1-\beta$, we need set \(\delta = \frac{\beta}{mT}=\mathcal{O}\left({\beta}\varepsilon^{2}\right)\) and sample size \({\mathcal{O}}\left(\frac{\log^3\frac{1}{\varepsilon}\log\frac{1}{\ell}}{\mu_F(1-\gamma)^{10+12m} \ell^{6m} \varepsilon^6}\right)\). 
\end{proof}
\section{Experiment}
\label{expri}
We conducted experiments \footnote{All the experiments in this subsection were carried out on a MacBook Pro with an Apple M1 Pro chip and 32 GB of RAM. Our code is developed based on \cite[https://github.com/andrschl/cirl]{schlaginhaufen2023identifiability}.} in a $6 \times 6$ gridworld environment introduced by \cite{sutton2018reinforcement} (see Figure \ref{policy}). We aim to reach the rewarded cell while controlling the time spent visiting the red rectangles under a certain threshold. We define the CMDP as follows: the environment involves four actions: \textit{up}, \textit{right}, \textit{down}, and \textit{left}. The agent moves in the specified direction with a $0.9$ probability and randomly selects another direction with a $0.1$ probability after taking an action. The constraints are defined as follows: If the agent hits the second-row or the fourth-row red rectangles, the reward functions $r_1(s,a)$ and $r_2(s,a)$ receive $-10$ respectively.
Once the agent reaches a rewarded cell, it remains there indefinitely, receiving a reward of $1$ per iteration. We set the discount factor $\gamma$ to $0.7$. We define the CMDP we solve as follows:
\[
\max_{\pi_\theta} V_{0}^{\pi_\theta}(\rho) \quad \text{subject to} \quad V_{i}^{\pi_\theta}(\rho)\ge -2, \quad i \in [2],
\]
utilizing the softmax policy parameterization.
\begin{figure}[H]
\begin{center}
      \includegraphics[scale = 0.17]{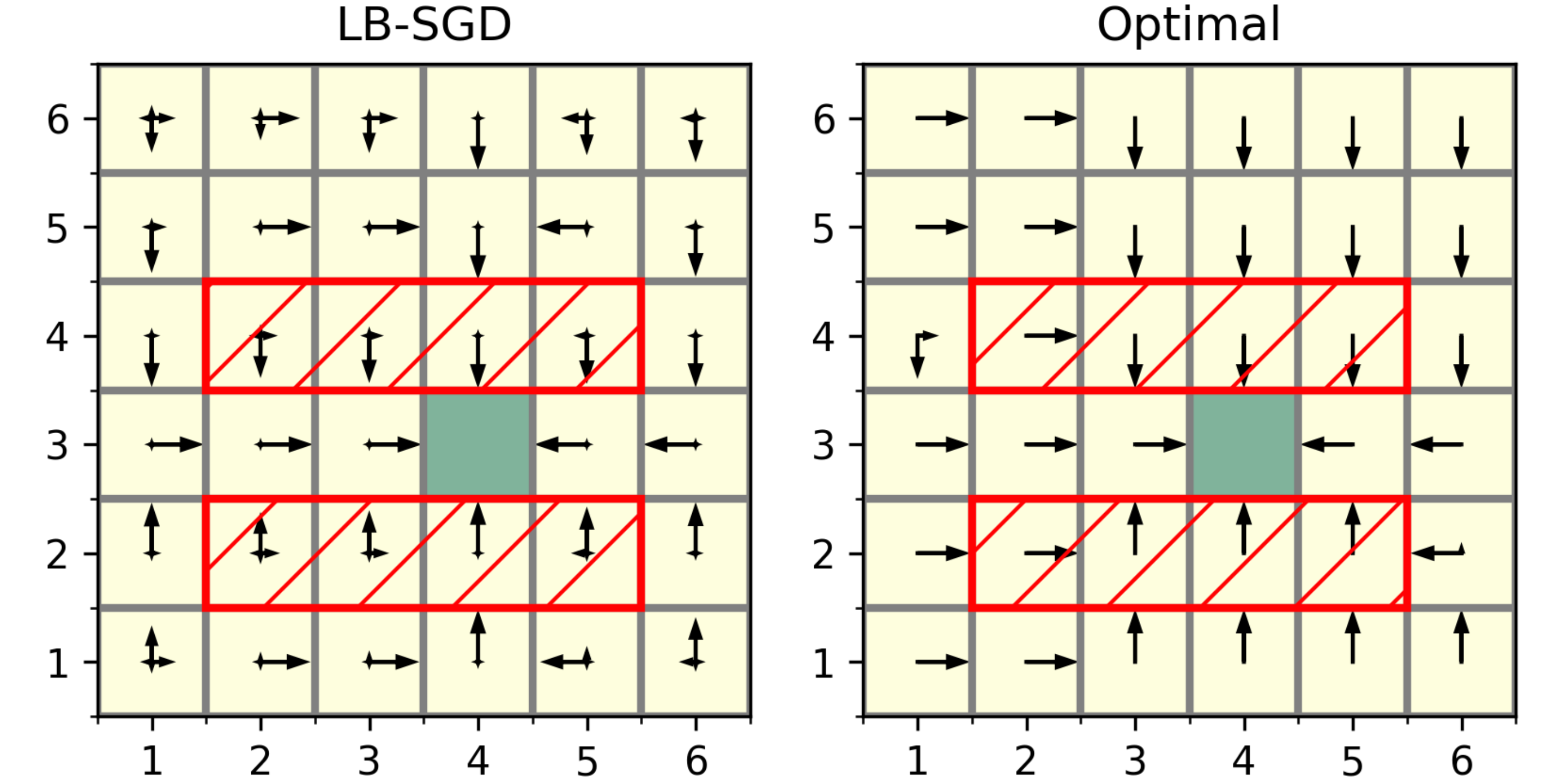}
    \caption{Gridworld Environment: The green block denotes the reward, the arrows represent the policies, and the two red-hatched rectangles indicate the constrained states.}
    \label{policy}
\end{center}
\end{figure}

The optimal policy, depicted as arrows in Figure \ref{policy}, is computed using linear programming in occupancy-measure space. Notice that the agent needs to learn the optimal policy within the regions highlighted by red rectangles, as there is always a small probability of ending up there. 

 Since our paper emphasizes the feasibility of iterates rather than bounding average constraint violations, we compare with two typical CMDP learning algorithms: (1) the IPO algorithm \cite{liu2020ipo}, which combines the log barrier method with a policy gradient approach and fixed stepsize, and (2) the RPG-PD algorithm \cite{ding2024last}, a primal-dual method that guarantees the feasibility of the last iterate, using an entropy-regularized policy gradient and quadratic-regularized gradient ascent for the dual variable.

\subsection{Hyperparameter Tuning}
\label{para}
In this section, we outline the parameter choices for the three algorithms: LB-SGD, IPO, and RPG-PD.

For the LB-SGD algorithm, the Lipschitz constants \( M_g \) and the smoothness parameter \( M_h \) of the function \(\log \pi_\theta(a|s)\) are required to compute the smoothness of the value functions \(V_i^\theta(\rho)\), which are critical for determining stepsizes. For direct, softmax, or log-linear parameterizations, these values can be computed directly. For other parameterizations, they can be estimated from sampled trajectories (see Appendix \ref{estsm}). Although \( M_g \) and \( M_h \) are both 1 for the softmax parameterization used in our experiment, we validated the effectiveness of the estimation approach described in Appendix \ref{estsm}. To implement the LB-SGD algorithm, we further need to choose the parameter \( \eta \), which needs to satisfy \( \eta \le \nu_{\text{emf}} \) according to Theorem \ref{mainn}. However, calculating \( \nu_{\text{emf}} \) is generally infeasible. In practice, we do not need its exact value to select \( \eta \). If \( \eta \ge \nu_{\text{emf}} \), the iterates remain within a distance of \( \Omega(\nu_{\text{emf}}) \) from the boundary rather than \( \Omega(\eta) \). Therefore, LB-SGD ensures safe exploration for any \( \eta \), though smaller \( \eta \) values bring the iterates closer to the boundary, expanding the search area for policy optimization. A smaller \( \eta \) leads to a final policy closer to the optimal one, but this requires more samples to ensure accurate estimators, especially near the boundary. In our case, we selected \( \eta = 0.01 \) to keep the output policy close to optimal. After fixing $\eta$, we need to choose the number of trajectories per iteration for calculating the stepsize and the log barrier gradient. In Figure \ref{lblearning}, we plot the gradient estimation error for the log barrier function and the stepsize returned by the LB-SGD algorithm, using different sample sizes at different points that vary the distance to the boundary. As shown in Figure \ref{lblearning}, when the point is closer to the boundary, our algorithm requires a higher number of samples per iteration to obtain accurate estimates of stepsizes and log barrier gradients. Inadequate sampling leads to relatively smaller estimates of stepsizes with higher variance in both gradient and stepsize estimations. Therefore, we used 3,000 trajectories per iteration to accurately estimate the gradient and stepsize.
\begin{figure}[t]
            \centering
            \makebox[\textwidth][c]{%
                \resizebox{1.2\textwidth}{!}{%
                    \includegraphics{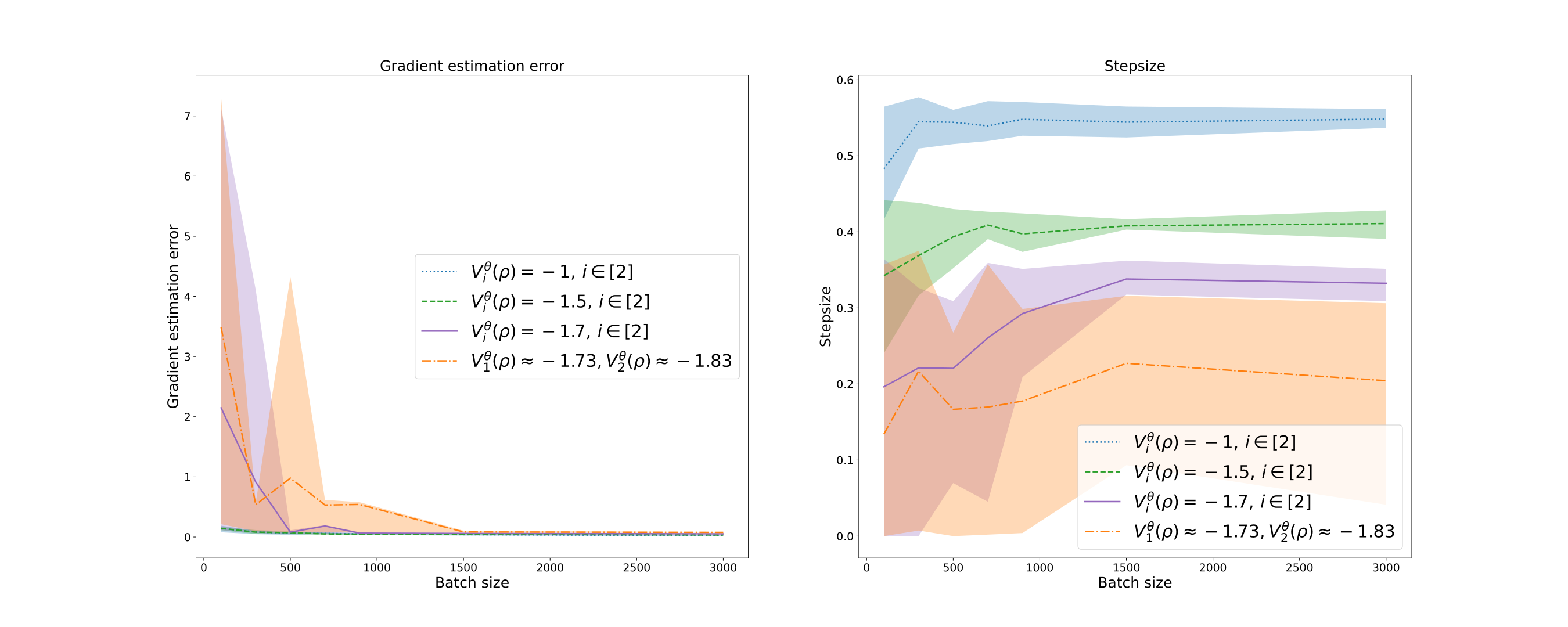}
                }
            }
                \caption{The gradient estimation error for the log barrier function and the value of stepsize for LB-SGD algorithm with sample sizes of 100, 300, 500, 700, 900, 1500, and 3000 at varying distances from the boundary. In the figure above, the lines indicate the median values obtained from 10 independent experiments, while the shaded areas represent the 10\% and 90\% percentiles calculated from 10 different random seeds.}
                \label{lblearning}
\end{figure}

Since the IPO algorithm also uses the log barrier method combined with a policy gradient approach, similar to LB-SGD, we chose the same value for \( \eta \) and used 3,000 trajectories per iteration to estimate the gradient. The main difference lies in IPO's manually chosen fixed stepsize. To ensure safe exploration and avoid unsafe behavior from larger stepsizes, we selected stepsizes of 1.5, 1, and 0.5.

The RPG-PD algorithm requires tuning three parameters: the regularization parameter \( \tau \), the stepsize, and \( \varepsilon_0 \) for projecting the policy onto the probability simplex. We selected the stepsize based on \cite[Theorems 2]{ding2024last}, which provides an optimal choice of \( \eta \), determined by \( \tau \), \( \varepsilon_0 \), and the Slater parameter \( \nu_s \) which can be computed since we are given a safe initial starting point. To ensure the safety of the last iterate, RPG-PD learns from more conservative constraints, \( V_i^\pi(\rho) \ge b-2 \) with \( b > 0 \). We fine-tuned the parameters $\tau$, $\varepsilon_0$ and $b$ using exact information and selected \( \tau = \varepsilon_0= b = 0.1 \). We plotted the performance of RPG-PD with these parameters in Figure \ref{smp}, labeled as learning curve RPG-PD\_exact. For the stochastic version of RPG-PD, we used the same parameters as the exact version and employed 800 trajectories to estimate gradient information.

\subsection{Comparison and conclusion}
In this section, we analyze the performance of the three algorithms—LB-SGD, IPO, and RPG-PD, using parameters as described in Section \ref{para}, and evaluated with stochastic information.

In terms of ensuring safe exploration, all three algorithms successfully achieve this, as shown in Figure \ref{smp}. However, for the IPO and RPG-PD algorithms, stepsize tuning is crucial to avoid constraint violations during learning, as larger stepsizes can lead to unsafe behavior. In contrast, the LB-SGD algorithm, which adapts its stepsizes based on sample estimates, eliminates the need for manual tuning.

Regarding convergence guarantees, LB-SGD proves to be more sample-efficient in finding the optimal policy compared to RPG-PD. By iteration 4000, the RPG-PD algorithm's iterate remains approximately 0.5 away from the boundary, whereas LB-SGD’s iterate is only about 0.02 away. This result is consistent with theoretical results: RPG-PD requires \( \mathcal{O}(\varepsilon^{-14}) \) samples to find an \( \mathcal{O}(\varepsilon) \)-optimal policy, while LB-SGD requires only \(\Tilde{ \mathcal{O}}(\varepsilon^{-6}) \) samples. However, the IPO algorithm, with well-tuned stepsizes, converges faster and achieves closer proximity to the optimal reward value compared to LB-SGD. This is due to LB-SGD’s more conservative choice of stepsizes to ensure safe exploration. Specifically, the variation in constraint 2 values is smaller with LB-SGD than with IPO, particularly when the initial point is near the boundary.

Our experiment confirmed that LB-SGD achieves safe exploration while converging to the optimal policy efficiently. Meanwhile, in Figure \ref{policy}, we depict the policy obtained from our algorithm. Although it bears a resemblance to the optimal policy, it is less deterministic to circumvent the red rectangles. As expected, ensuring these guarantees necessitates a higher number of samples per iteration near the boundary for accurate estimates (see Figure \ref{lblearning}). It would be interesting to determine whether this sample complexity is inherent to our algorithm and its analysis or if it is a consequence of the safe exploration requirement.

\subsection{Discussion}
Due to approximation errors, all algorithms—including IPO, RPG-PD, and LB-SGD—may occasionally take a bad step and produce an infeasible iterate \( \pi_{\theta_t} \). While sometimes the gradient update remains feasible, allowing the algorithm to recover from its bad step automatically, in other cases, a recovery method becomes necessary. In our experiments, we implement recovery by reverting to previous iterates and increasing the sample complexity while simultaneously decreasing the stepsize at previous iterates. This ensures the safe exploration of subsequent iterates with high probability.

\begin{figure}[t]
            \centering
            \makebox[\textwidth][c]{%
                \resizebox{1.2\textwidth}{!}{%
                    \includegraphics{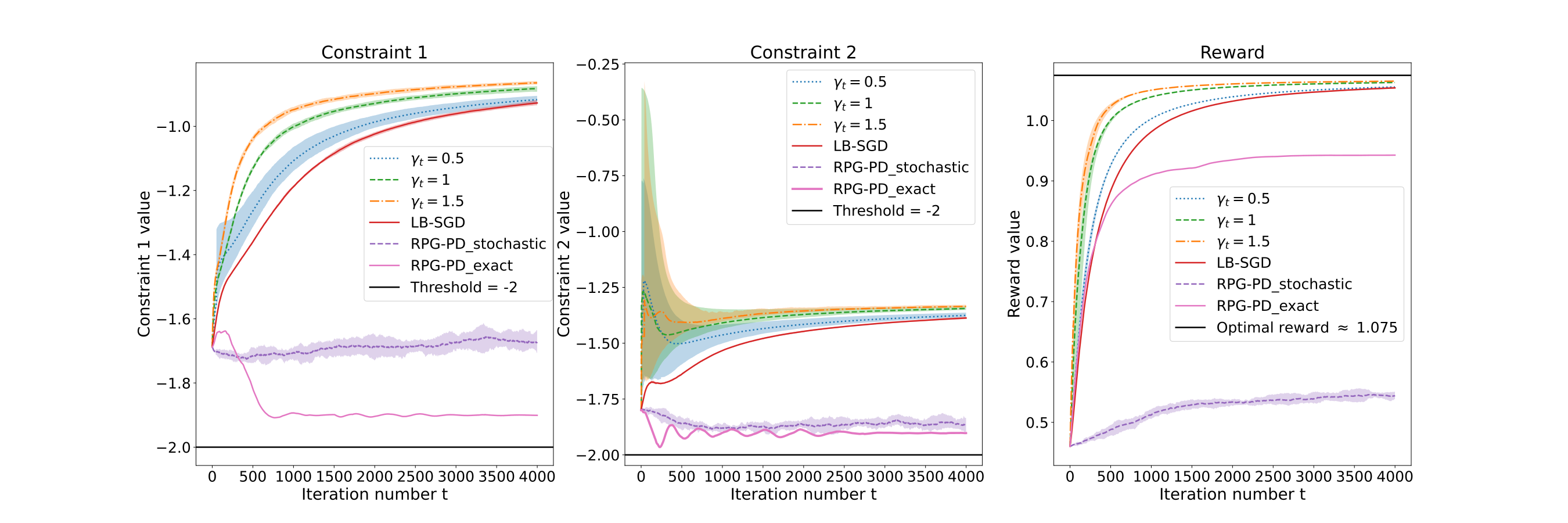}
                }
            }
            \caption{The average performance comparison between the IPO algorithm \cite{liu2020ipo} with stepsizes \( \gamma_t = 0.5, 1, 1.5 \), RPG-PD using a regularization parameter \( \tau = 0.1 \) and stepsize chosen according to \cite[Theorem 2]{ding2024last} with \( b = 0.1 \), and LB-SGD is shown below. The lines represent the median values from 10 independent experiments, while the shaded areas illustrate the 10\% and 90\% percentiles calculated from 10 different random seeds.}
          \label{smp}
\end{figure}
\section{Estimation of smoothness parameter}
\label{estsm}
In \cite[Proof of Lemma 4.4]{yuan2022general}, the second order of the value functions $\nabla^2 V_i^\theta(\rho)$ is computed as
\begin{align*}
    &\nabla^2 V_i^\theta(\rho)\\  =&\mathbb{E}_{\tau\sim\pi_\theta}\Biggl[\sum_{t=0}^{\infty}\gamma^t r_i(s_t,a_t)\Biggl(\left(\sum_{k=0}^{t}\nabla_\theta^2\log \pi_\theta(a_k|s_k)\right)+\left(\sum_{k=0}^{t}\nabla_\theta\log \pi_\theta(a_k|s_k)\right)\left(\sum_{k=0}^{t}\nabla_\theta\log \pi_\theta(a_k|s_k)\right)^T\Biggr)\Biggr].
\end{align*}
\normalsize
Using the information of $n$ truncated trajectories, with a fixed horizon $H$, denoted as $\tau_j := \left(s_t^j, a_t^j, \left\{r_{i}(s_t^j, a_t^j)\right\}_{i=0}^{m}\right)_{t=0}^{H-1}$, we can estimate $\nabla^2 V_i^\theta(\rho)$ by the Monte-Carlo method as
\begin{align*}
    &\hat \nabla^2 V_i^\theta(\rho)\\
    =&\frac{1}{n}\sum_{j=1}^{n}\Biggl[\sum_{t=0}^{H}\gamma^t r_i(s_t^j,a_t^j)\Biggl(\left(\sum_{k=0}^{t}\nabla_\theta^2\log \pi_\theta(a_k^j|s_k^j)\right)+\left(\sum_{k=0}^{t}\nabla_\theta\log \pi_\theta(a_k^j|s_k^j)\right)\left(\sum_{k=0}^{t}\nabla_\theta\log \pi_\theta(a_k^j|s_k^j)\right)^T\Biggr)\Biggr].
\end{align*}
\normalsize
We estimate the smoothness parameter $M_i$ for the value function $V_i^\theta(\rho)$ as
\begin{align*}
    M_i= \|\hat \nabla^2 V_i^\theta(\rho)\|.
\end{align*}

\end{document}